\documentclass{article}
\usepackage[margin=1in]{geometry}
\usepackage[hidelinks]{hyperref}
\hypersetup{
    colorlinks=true,
    linkcolor=blue,
    citecolor=blue,
    urlcolor=blue,
}

\makeatletter
\patchcmd\@footnotemark{{link}}{{footnote}}{}{\fail}
\makeatother
\ActivateGenericHook{hyp/link/footnote}
\AddToHook{hyp/link/footnote}{\hypersetup{linkcolor=black}}

\usepackage{graphicx} 
\usepackage{subcaption}
\usepackage{caption}
\usepackage{amsmath}
\usepackage{stmaryrd}
\usepackage{amssymb}
\usepackage{amsfonts}
\usepackage{float}
\usepackage{tikz}
\usetikzlibrary{positioning}
\usepackage{algorithm}
\usepackage{algpseudocode}
\usepackage{diagbox}
\usepackage{amsthm}
\usepackage{mathrsfs} 
\usepackage{mathtools}
\usepackage{setspace}
\usepackage{xcolor}

\newtheorem{theorem}{Theorem}
\newtheorem*{theorem*}{Theorem}
\newtheorem{proposition}{Proposition}
\newtheorem*{proposition*}{Proposition}
\newtheorem{corollary}{Corollary}
\newtheorem*{corollary*}{Corollary}

\theoremstyle{definition}
	
\newtheorem{remark}{Remark}

\newcommand{\cC}{\mathcal{C}}

\newcommand{\cI}{\mathcal{I}}

\newcommand{\cS}{\mathcal{S}}

\newcommand{\cX}{\mathcal{X}}
\newcommand{\cY}{\mathcal{Y}}

\newcommand{\N}{\mathbb{N}}
\newcommand{\R}{\mathbb{R}}

\DeclareMathOperator*{\argmin}{arg\,min}

\renewcommand{\epsilon}{\varepsilon}

\algnewcommand\Input{\item[\textbf{Input:}]}
\algnewcommand\Output{\item[\textbf{Output:}]}

\usepackage[symbol*]{footmisc}

\begin{document}

\begin{center}

    \rule{\linewidth}{2pt}
    
    \vspace{0.5cm}
    
    {\Large \bfseries Cluster-Aware Matching via Laplacian Optimal Transport}
    
    \vspace{0.3cm}
    
    \rule{\linewidth}{0.7pt}
    
    \vspace{1cm}
    
    \textbf{Gabriel Samberg}, \textit{Department of Applied Mathematics, Tel Aviv University} \\
    \textbf{YoonHaeng Hur}\footnote{Corresponding author: yoonhaeng.hur@columbia.edu}, \textit{Department of Statistics, Columbia University} \\
    \textbf{Yuehaw Khoo}, \textit{Department of Statistics, University of Chicago} \\
    \textbf{Nir Sharon}, \textit{Department of Applied Mathematics, Tel Aviv University}

\end{center}

\vspace{0.5cm}

\begin{abstract}
    In many applications of matching, the point clouds to be matched are not merely unstructured sets of points but rather samples from distributions with an intrinsic cluster structure. In such cases, as individual points are often interchangeable within a coherent region, finding a robust region-to-region alignment is more desirable than establishing a precise point-to-point correspondence. To this end, we propose a novel approach for cluster-aware matching based on Laplacian Optimal Transport (LapOT). The key idea is to regularize the optimal transport problem with quadratic Laplacian terms constructed from similarity graphs of the point clouds, which encourages the optimal coupling to respect the cluster structure of both point sets. We also introduce Refined Simultaneous Clustering (RSC), a method that leverages the cluster-aware coupling obtained from LapOT to produce consistent partitions across the point sets, which can overcome the limitations of independent clustering and yield more stable and interpretable results. We demonstrate the effectiveness of our approach through theoretical analysis and empirical experiments, showing that LapOT indeed produces cluster-aware matching that leads to more consistent and meaningful alignments between point clouds.
\end{abstract}

\section{Introduction} 
Matching is a fundamental problem of establishing meaningful correspondences or alignment \cite{van2011survey,sahilliouglu2020recent} between two or more sets of points drawn from potentially distinct distributions or geometric spaces. Traditionally rooted in computer vision and pattern recognition for rigid and non-rigid shape alignment \cite{besl1992method,myronenko2010point}, the scope of point matching has expanded significantly with the advent of high-dimensional statistics and machine learning. Importantly, along with the rise of closely related optimal transport theory \cite{villani_2003} and computation \cite{peyre2019computational}, establishing accurate point-to-point correspondences is increasingly critical in many applications, such as computational biology \cite{schiebinger2019optimal,blumberg2020mrec}, biological imaging \cite{riahi2023alignot,singer2024alignment,xiao2025optimal}, word translation \cite{alvarez2018gromov,grave2019unsupervised}, to name a few.

In many applications, the point clouds to be matched are not merely unstructured sets of points but rather samples from distributions with an intrinsic cluster structure. These clusters often represent meaningful regions or components of the underlying objects, such as the limbs of a human figure or functional groups of proteins. In such cases, as individual points are often interchangeable within a coherent region, finding a robust region-to-region alignment is more desirable than establishing a precise point-to-point correspondence. A natural yet naive approach to achieve such a region-level alignment would be to first cluster each point cloud independently and then match the resulting clusters. However, this two-stage approach can be extremely problematic due to the instability of clustering algorithms \cite{ben2006sober,rakhlin2006stability,von2010clustering}, leading to inconsistent partitions across the point sets, which in turn severely degrades the quality of the subsequent matching. This instability is illustrated in Figure \ref{fig:RSC_motivating}(a), where independent clustering fails to produce consistent partitions for two similar shapes. This motivates the need for a cluster-aware matching framework that integrates information about the cluster structure directly into the matching process. By treating matching and clustering as coupled problems, we can leverage the shared structure of the point clouds to obtain more robust and meaningful correspondences.

To this end, we propose a novel approach for cluster-aware matching based on Laplacian Optimal Transport (LapOT). The key idea is to regularize the optimal transport problem with quadratic Laplacian terms that encourage the optimal coupling to respect the cluster structure of both point sets. This is achieved by constructing similarity graphs for each point cloud, where the edges encode the likelihood of points belonging to the same cluster. The resulting Laplacian regularization terms promote similar rows and columns in the optimal coupling for points that are close in their respective similarity graphs, effectively encouraging a cluster-aware matching. We demonstrate the effectiveness of our approach through theoretical analysis and empirical experiments, showing that LapOT indeed produces cluster-aware couplings that lead to more consistent and meaningful alignments between point clouds.

We also introduce a Refined Simultaneous Clustering (RSC) method that leverages the cluster-aware coupling obtained from LapOT to produce consistent partitions across the point sets. By conditioning the clustering process on the information provided by the matching, RSC can overcome the limitations of independent clustering and yield more stable and interpretable results. Figure \ref{fig:RSC_motivating}(b) illustrates the improved clustering obtained by RSC, which produces consistent and meaningful clusters across both shapes.

The rest of the paper is organized as follows. In Section \ref{sec:lapot}, after laying out the relevant background, we introduce the LapOT framework. Section \ref{sec:theory} establishes theoretical guarantees for the cluster-aware properties of the optimal coupling obtained from LapOT. In Section \ref{sec:rsc}, we present the RSC method and demonstrate its effectiveness through experiments on 3D shapes, followed by further applications in Section \ref{sec:applications}.

\begin{figure}[!htbp]
  \centering
  \subfloat[Independent Clustering]{
    \includegraphics[width=0.48\textwidth, trim=0.6in 0.6in 0.6in 0.6in, clip]{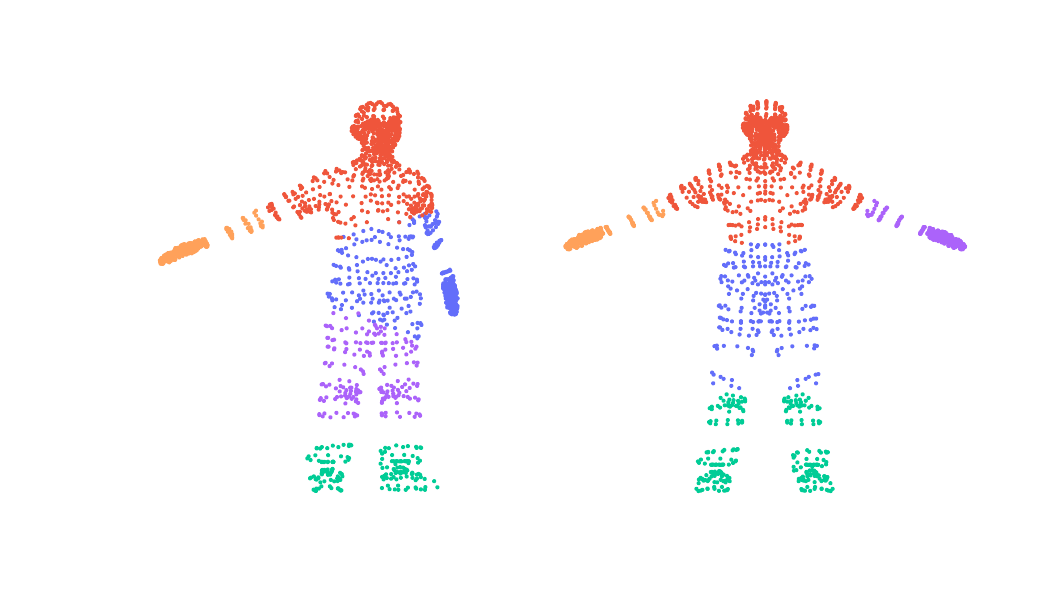}
  }
  \subfloat[Our Method (RSC)]{
    \includegraphics[width=0.48\textwidth, trim=0.6in 0.6in 0.6in 0.6in, clip]{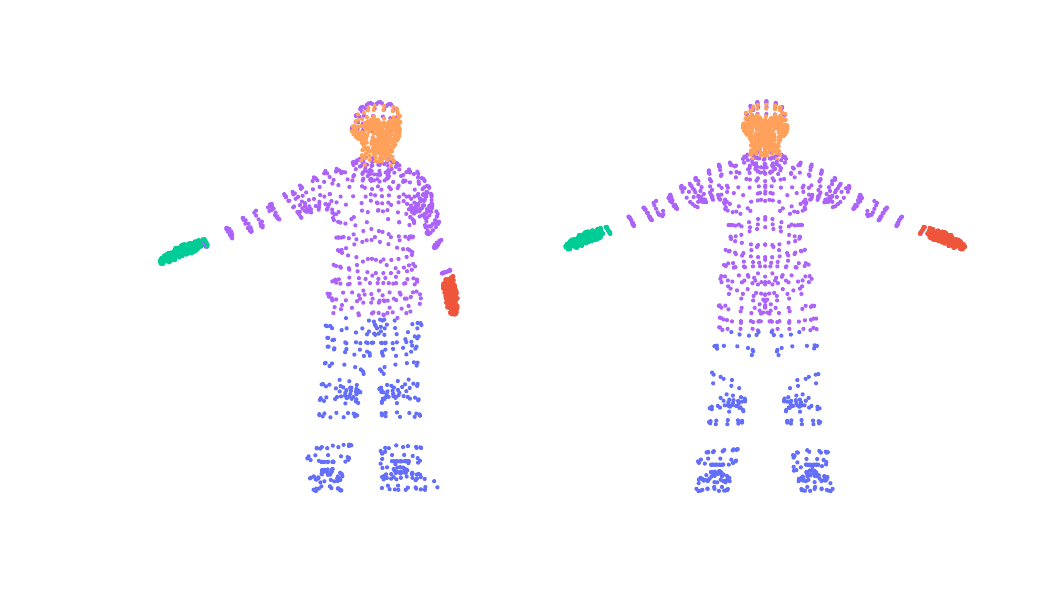}
  }
  \caption{Clustering of two 3D human shapes from the CAPOD dataset \cite{papadakis2014canonically}. (a) is obtained by applying spectral clustering independently to the two shapes with $k = 5$ for each shape, which yields inconsistent partitions. (b) is the result of our proposed Refined Simultaneous Clustering (RSC) method, which produces consistent and meaningful clusters across both shapes learned from the cluster-aware matching obtained via Laplacian optimal transport. See Section \ref{sec:rsc} for details.}
  \label{fig:RSC_motivating}
\end{figure}

\paragraph{Notation.}
Let $\mathrm{tr}(A)$ denote the trace of a square matrix $A$. For matrices $A, B$ of the same dimension, let $\langle A, B \rangle \triangleq \mathrm{tr}(A^\top B)$ denote their Frobenius inner product, with the associated Frobenius norm $\|A\|_{\mathrm{F}} \triangleq \sqrt{\langle A, A \rangle}$. For a matrix $A \in \R^{n \times m}$, let $\|A\|_1 \triangleq \sum_{i, j} |A_{i j}|$ and $\|A\|_\infty = \max_{i,j} |A_{i j}|$. For a vector $v \in \R^n$ and a symmetric positive semi-definite matrix $L \in \R^{n \times n}$, let $\|v\|_{L} \triangleq \sqrt{v^\top L v}$. Let $\R_+ = [0,\infty)$. For any $n \in \N$, let $\cS_n$ be the set of all permutations of $[n] = \{1, \ldots, n\}$, let $1_n = (1, \ldots, 1) \in \R^n$ denote the all-ones vector, and let $\Delta_n \triangleq \{a \in \R_+^n : \sum_{i = 1}^n a_i = 1\}$ be the probability simplex. For $a \in \Delta_n$ and $b \in \Delta_m$, we denote by $\Pi_{a, b} \triangleq \{P \in \R_+^{n \times m} : P 1_m = a ~ \text{and} ~ P^\top 1_n = b\}$ the set of couplings between $a$ and $b$. For any $P \in \R_+^{n \times m}$, let $H(P) \triangleq - \sum_{i, j} P_{i j} (\log(P_{i j}) - 1)$ denote the Shannon entropy; we also use $H(v) \triangleq - \sum_{i} v_i (\log(v_i) - 1)$ to denote the Shannon entropy of a vector $v \in \R_+^n$. For two vectors $u, v \in \R^n$, let $u / v$ denote the element-wise division of $u$ by $v$. For a vector $v \in \R^n$, let $\mathrm{diag}(v)$ denote the diagonal matrix with $v$ on the diagonal. For any $z \in \R^d$, let $\delta_z$ denote the Dirac measure at $z$.

\section{Laplacian Optimal Transport}
\label{sec:lapot}
Point matching concerns finding a correspondence between two sets of points $\{X_1, \ldots, X_n\}$ and $\{Y_1, \ldots, Y_m\}$ defined on some sets $\cX$ and $\cY$, respectively. This section introduces Laplacian Optimal Transport (LapOT), a novel approach for point matching that incorporates the cluster structure of the point sets. We first review relevant background on point matching and optimal transport, followed by a detailed description of LapOT and its low-rank extension. 

\subsection{Preliminaries: Matching, Optimal Transport, and Clustering}
\label{sec:preliminaries}

\paragraph{Quadratic Matching.}
Quadratic matching is one standard approach for point matching, which seeks a suitable alignment between two sets of points based on pairwise relationships. To this end, we first define two pairwise relationship matrices $A \in \R^{n \times n}$ and $B \in \R^{m \times m}$, which capture the relationships between the points in $\{X_1, \ldots, X_n\}$ and $\{Y_1, \ldots, Y_m\}$, respectively. The choice of $A$ and $B$ depends on the specific application and the nature of the data. For instance, when $\cX$ is a metric space with a metric $d_\cX$, we may set $A_{i i'} = d_\cX(X_i, X_{i'})$; when the points are from a graph, we may use the adjacency matrix as $A$. Once suitable $A, B$ are chosen, quadratic matching seeks the best alignment between them. Particularly, when $n = m$, this can be formulated as follows:
\begin{equation}
    \label{eq:quadratic_matching}
    \min_{\sigma \in \cS_n} \sum_{i, i' = 1}^{n} |A_{i i'} - B_{\sigma(i) \sigma(i')}|^p,
\end{equation}
where $p > 0$ is a suitable exponent. In words, \eqref{eq:quadratic_matching} finds the permutation that minimizes the total discrepancy between $A$ and $B$ after alignment. When $n \neq m$, we may seek the best alignment in terms of couplings, leading to the following formulation:
\begin{equation}
    \label{eq:gw}
    \min_{\pi \in \Pi_{a, b}} \sum_{i, i' = 1}^{n} \sum_{j, j' = 1}^{m} |A_{i i'} - B_{j j'}|^p \pi_{i j} \pi_{i' j'},
\end{equation}
where $a \in \Delta_n$ and $b \in \Delta_m$ are user-specified weights to emphasize the importance of the points. In plain language, this formulation seeks a coupling that aligns $A, B$ with the smallest total discrepancy. When $\cX, \cY$ are metric spaces with $A, B$ being the pairwise distances of the points, \eqref{eq:gw} is the Gromov-Wasserstein distance \cite{memoli2011gromov} between two empirical measures $\sum_{i = 1}^{n} a_i \delta_{X_i}$ and $\sum_{j = 1}^{m} b_j \delta_{Y_j}$. Despite its rich theoretical properties \cite{memoli2024comparison,dumont2025existence} and practical success across various domains \cite{peyre2016gromov,solomon2016entropic,alvarez2018gromov,xu2019gromov,bunne2019learning,hur2024reversible}, quadratic matching is NP-hard: notice that \eqref{eq:quadratic_matching} is a quadratic assignment problem \cite{koopmans1957assignment,loiola2007survey}, and \eqref{eq:gw} is a non-convex quadratic program.

\paragraph{Profile-Based Matching and Optimal Transport.}
Another approach to point matching is to assign suitable profiles to the points and match them if the profiles are similar, which has been considered by \cite{ding2021efficient} and \cite{hur2025robust} for graph and object matching, respectively. Concretely, we assign to $X_i$ the weighted distribution $\mu_i \triangleq \sum_{i' = 1}^{n} a_{i'} \delta_{A_{i i'}}$ as its profile, called a degree profile in \cite{ding2021efficient} when $A$ is the adjacency matrix and a distance profile in \cite{hur2025robust} when $A$ consists of the pairwise distances. Similarly, $Y_j$'s profile is defined as $\nu_j \triangleq \sum_{j' = 1}^{m} b_{j'} \delta_{B_{j j'}}$. Here, $a_i$'s and $b_j$'s are user-specified weights to emphasize the importance of the points. Now that the profiles are one-dimensional measures, we construct a matrix $C \in \R^{n \times m}$ whose entries measure the similarity between the profiles based on a suitable discrepancy, for instance, $C_{i j} = W_1(\mu_i, \nu_j)$ when using the Wasserstein distance. Then, the point matching problem can be reduced to the following optimization problem:
\begin{equation}
    \label{eq:discrete_ot}
    \min_{\pi \in \Pi_{a, b}} \langle \pi, C \rangle,
\end{equation}
which is a discrete optimal transport (OT) problem between $\sum_{i = 1}^{n} a_i \delta_{X_i}$ and $\sum_{j = 1}^{m} b_j \delta_{Y_j}$, with $C$ serving as a transport cost matrix. Unlike quadratic matching, \eqref{eq:discrete_ot} is a linear program that can be solved by polynomial-time algorithms. Of course, solving \eqref{eq:discrete_ot} can still be computationally expensive for large-scale problems. To tackle this, the following entropic regularized optimal transport problem \cite{wilson1969use,cuturi2013sinkhorn,peyre2019computational} has been proposed: given any $a \in \Delta_n$, $b \in \Delta_m$, and $C \in \R^{n \times m}$, 
\begin{equation}
    \label{eq:entropic_ot}
    \mathsf{OT}_{\epsilon}(a, b, C) \triangleq \min_{\pi \in \Pi_{a, b}} \langle \pi, C \rangle -\epsilon H(\pi).
\end{equation}
Adding the entropic regularization term to \eqref{eq:discrete_ot} leads to a strongly convex problem. Crucially, \eqref{eq:entropic_ot} can be solved efficiently via the Sinkhorn algorithm that iteratively scales the kernel matrix $e^{-C / \epsilon}$ so that its row and column sums become $a$ and $b$, respectively.

\paragraph{Clustering via Graph Cuts.} 
One common approach to clustering is to view the data as a graph and partition the graph into clusters by minimizing the graph cuts. To be specific, consider an undirected graph with $n$ vertices with an adjacency matrix $K \in \R^{n \times n}$. Here, the idea is to cluster the vertices of the graph by minimizing the total weight of edges that are cut by the partition. Given a partition $\cI_1 \cup \cdots \cup \cI_k = [n]$ of the vertices, each partition $\cI_j$ represents a cluster, and the total weight of edges that are cut by $\cI_j$ is given by $\sum_{i \in \cI_j, i' \notin \cI_j} K_{i i'}$. Hence, the sum of these weights for all partitions, weighted by the size of the partitions, can be written as
\begin{equation*}
    \sum_{j = 1}^{k} \frac{\sum_{i \in \cI_j, i' \notin \cI_j} K_{i i'}}{|\cI_j|},
\end{equation*}
which is called the ratio cut of the partition. This can be rewritten as $\langle P, L P \rangle$, where $L = \mathrm{diag}(K 1_n) - K$ is the unnormalized graph Laplacian and $P \in \R^{n \times k}$ is the partition matrix defined by $P_{i j} = \frac{1}{\sqrt{|\cI_j|}}$ if vertex $i$ belongs to partition $\cI_j$ and $P_{i j} = 0$ otherwise. Accordingly, minimizing the ratio cut can be formulated as the following optimization problem:
\begin{equation*}
    \min_{P \in \R^{n \times k}} \langle P, L P  \rangle, \quad \text{subject to} ~ P_{i j} 
    = 
    \begin{cases}
        \frac{1}{\sqrt{|\cI_j|}} & \text{if vertex} ~ i ~ \text{belongs to partition} ~ \cI_j, \\
        0 & \text{otherwise}.
    \end{cases}
\end{equation*}
This, however, is a combinatorial optimization problem that is NP-hard. A common relaxation is to relax the constraint on $P$ to be $P^\top P = I_k$, which leads to the following optimization problem:
\begin{equation*}
    \min_{P \in \R^{n \times k}} \langle P, L P \rangle, \quad \text{subject to} ~ P^\top P = I_k,
\end{equation*}
which boils down to finding the $k$ eigenvectors of $L$ corresponding to the $k$ smallest eigenvalues. This is the backbone of spectral clustering \cite{belkin2001laplacian,coifman2006diffusion,von2007tutorial}, which is a widely used clustering algorithm in practice. Instability of clustering \cite{ben2006sober,rakhlin2006stability,von2010clustering} can lead to unreliable scientific findings and downstream decisions filled with uncertainty \cite{senbabaouglu2014critical, smith2025lack}, which has motivated a huge body of work on uncertainty quantification for clustering \cite{kerr2001bootcluster, hofmans2015bootstrapkmeans,liu2018ukmeans,huang2015statistical,wade2018bayesian,hur2026inference,nath2026weighted}. However, these methods are not designed to produce consistent partitions across multiple point sets, which is crucial for cluster-aware matching. This motivates the need for a new approach that can leverage the shared structure of the point clouds to obtain more robust and meaningful correspondences.

\subsection{Laplacian Optimal Transport (LapOT)}
We introduce Laplacian Optimal Transport (LapOT), a novel approach for point matching that incorporates the cluster structure of the point sets. The key idea is to regularize the optimal transport problem \eqref{eq:entropic_ot} with quadratic Laplacian regularization terms that encourage the optimal coupling to carry the cluster structure of both point sets. 

To this end, we require the user to specify the cluster structure of the point sets in the form of similarity graphs, represented by the symmetric matrices $K_X \in \R^{n \times n}$ and $K_Y \in \R^{m \times m}$, for the point sets $\{X_1, \ldots, X_n\}$ and $\{Y_1, \ldots, Y_m\}$, respectively. The entry $(K_X)_{i i'}$ (resp.\ $(K_Y)_{j j'}$) represents the similarity between $X_i$ and $X_{i'}$ (resp.\ $Y_j$ and $Y_{j'}$), where the larger value indicates higher similarity and thus being more likely to belong to the same cluster. The choice of $K_X$ and $K_Y$ depends on the specific application and the nature of the data. For instance, when $\cX$ is a metric space with a metric $d_\cX$, we may set $(K_X)_{i i'} = e^{-d_\cX(X_i, X_{i'}) / \sigma}$ for some $\sigma > 0$; when the points are from a graph, we may use the usual binary adjacency matrix as $K_X$. The same applies to $K_Y$ for the point set $\{Y_1, \ldots, Y_m\}$.

Now, we want the optimal coupling to reflect the cluster structure encoded by $K_X$ and $K_Y$. Concretely, when $(K_X)_{i i'}$ (resp.\ $(K_Y)_{j j'}$) is large, we want to encourage the $i$-th and $i'$-th rows (resp.\ $j$-th and $j'$-th columns) of the optimal coupling to be similar to each other. In other words, if $X_i$ and $X_{i'}$ are similar to each other, we want them to be matched to similar distributions over $\{Y_1, \ldots, Y_m\}$; the same applies to $Y_j$'s. To do so, we define $L_X = \mathrm{diag}(K_X 1_n) - K_X$ and $L_Y = \mathrm{diag}(K_Y 1_m) - K_Y$, which are the unnormalized graph Laplacians of the similarity graphs defined by $K_X$ and $K_Y$, respectively. Then, for a coupling $\pi \in \Pi_{a, b}$, consider the following terms:
\begin{equation}
    \label{eq:laplacian_terms}
    \begin{aligned}
        \langle \pi, L_X \pi \rangle & = \mathrm{tr}(\pi^\top L_X\pi) = \frac{1}{2} \sum_{i, i' = 1}^{n} (K_{X})_{i i'} \|\pi_i - \pi_{i'}\|_2^2, \\
        \langle \pi, \pi L_Y \rangle & = \mathrm{tr}(\pi L_Y \pi^\top) = \frac{1}{2}\sum_{j, j' = 1}^{m}(K_{Y})_{j j'}\|\pi^\top_j - \pi^\top_{j'}\|_2^2.
    \end{aligned}
\end{equation}
Here, $\pi_i$ and $\pi^\top_j$ are the $i$-th row and $j$-th column of the coupling $\pi$, respectively. By adding these terms to \eqref{eq:entropic_ot}, we encourage the optimal coupling to have similar rows for similar $X_i$'s and similar columns for similar $Y_j$'s. Accordingly, the optimal coupling is incentivized to carry the cluster structure of both point sets defined by $K_X$ and $K_Y$.

In summary, given any $a \in \Delta_n$, $b \in \Delta_m$, $C \in \R^{n \times m}$, and $L_X, L_Y$ defined as above, we define LapOT as follows:
\begin{equation}
    \label{eq:LapOT}
    \mathsf{LapOT}_{\lambda_x, \lambda_y, \lambda}(a, b, C, L_X, L_Y) \triangleq \min_{\pi \in \Pi_{a, b}} \langle \pi, C \rangle + \lambda_x \langle \pi, L_X \pi \rangle + \lambda_y \langle \pi, \pi L_Y \rangle - \lambda H(\pi),
\end{equation}
where $\lambda_x, \lambda_y, \lambda > 0$ are hyperparameters that control the strength of the regularization. We can view LapOT as a generalization of entropic OT: when $\lambda_x = \lambda_y = 0$, LapOT reduces to entropic OT. By tuning $\lambda_x, \lambda_y$, we can control the extent to which the optimal coupling is encouraged to carry the cluster structure of the point sets, which can be beneficial for downstream tasks such as clustering and alignment.

\paragraph{Optimization of LapOT.}
Note that the objective function of \eqref{eq:LapOT} consists of the convex quadratic terms---as the graph Laplacian is positive semidefinite---and the strongly convex entropic regularization term. Hence, \eqref{eq:LapOT} is a strongly convex optimization problem that can be solved by general-purpose convex optimization solvers. However, such solvers, which usually rely on interior point methods, can be inefficient for large-scale problems. \cite{hur2024convexified} proposed specialized algorithms for smooth objective functions with entropic regularization over couplings, which encompasses \eqref{eq:LapOT} as a special case. Their algorithms call the Sinkhorn algorithm as a subroutine, and the number of Sinkhorn calls is independent of the coupling size $n m$ if $\lambda$ is sufficiently large compared to the norm of $\lambda_x L_X + \lambda_y L_Y$ or depends logarithmically on $n m$ if $\lambda$ is sufficiently small (including the case $\lambda = 0$). Therefore, the algorithms are far more efficient than general-purpose convex optimization solvers, especially for large-scale problems, and we use them to solve \eqref{eq:LapOT} in our experiments. We refer the readers to \cite{hur2024convexified} for more details on the algorithms and their theoretical guarantees.

\begin{remark}
    The Laplacian regularization terms in \eqref{eq:laplacian_terms} were also considered in \cite{ferradans2014regularized} and \cite{courty2016optimal} in the context of domain adaptation. However, they do not study the cluster-awareness property of this regularization, which is the main focus of this paper. 
\end{remark}

\subsection{Low-Rank Extension of Laplacian Optimal Transport}
\label{sec:low_rank_LapOT}
As LapOT encourages the optimal coupling to carry the cluster structure of the point sets, it is natural to expect that the optimal coupling of LapOT enjoys a low-rank structure. While we will formally justify this in Section \ref{sec:theory}, we propose a method to directly impose a low-rank constraint on the optimal coupling of LapOT, which can be beneficial for large-scale problems. 

In the optimal transport literature, low-rank structure of couplings has been proposed and studied by \cite{forrow2019statistical,scetbon2021low,scetbon2022low} based on the following notion of non-negative rank:
\begin{equation}
    \label{eq:non_negative_rank}
    \mathrm{rk}_+(M) := \min\left\{r \in \N : \exists R_1, \ldots, R_r \in \R_+^{n \times m} ~ \text{s.t.} ~ \mathrm{rk}(R_1) = \cdots = \mathrm{rk}(R_r) = 1 ~ \text{and} ~ M = \sum_{i = 1}^{r} R_i\right\},
\end{equation}
where $\mathrm{rk}$ is the usual rank of a matrix. The idea of low-rank optimal transport is to solve the optimal transport problem by imposing a non-negative rank constraint on the coupling, namely, the set $\Pi_{a, b}$ of all couplings is replaced by the set of couplings with a non-negative rank of at most $r$ defined as 
\begin{equation*}    
    \Pi_{a, b}(r) \triangleq \{P \in \Pi_{a, b} : \mathrm{rk}_+(P) \le r\}.
\end{equation*}
The feasible set becomes non-convex, and the resulting optimization problem is no longer convex. In \cite{scetbon2021low}, instead of directly optimizing over $\Pi_{a, b}(r)$, the following observation is made:
\begin{equation}
    \label{eq:lr_ot}
    \Pi_{a, b}(r) = \bigcup_{g \in \Delta_r^+} \Pi_{a, g, b}
    \quad \text{where} \quad
    \Pi_{a, g, b} = \left\{U \mathrm{diag}(1_r / g) V^\top : U \in \Pi_{a, g} ~ \text{and} ~ V \in \Pi_{b, g}\right\},
\end{equation}
which leads to the following reformulation of the low-rank optimal transport problem:
\begin{equation*}
    \min_{\pi \in \Pi_{a, b}(r)} \langle \pi, C \rangle = \min_{(U, V, g) \in \cC(a, b, r)} \langle U \mathrm{diag}(1_r / g) V^\top, C \rangle,
\end{equation*}
where
\begin{equation}
    \label{eq:lr_feasible_set}
    \cC(a, b, r) = \{(U, V, g) \in \R_+^{n \times r} \times \R_+^{m \times r} \times \R_+^r : U 1_r = a, V 1_r = b, U^\top 1_n = V^\top 1_m = g\}.
\end{equation}
While the right-hand side of \eqref{eq:lr_ot} is non-convex as the objective function is not jointly convex in $U, V, g$, one can efficiently utilize Dykstra's algorithm \cite{dykstra1983algorithm} to solve the problem as a subroutine of a mirror descent algorithm with KL divergence. \cite{scetbon2021low} also considers the entropic regularized version of the low-rank optimal transport problem \eqref{eq:lr_ot} by adding individual entropic regularization terms for $U, V, g$ to the objective function, which leads to the following optimization problem:
\begin{equation*}
    \min_{(U, V, g) \in \cC(a, b, r)} \langle U \mathrm{diag}(1_r / g) V^\top, C \rangle - \lambda H(U) - \lambda H(V) - \lambda H(g).
\end{equation*}

Accordingly, a natural extension of LapOT to the low-rank setup can be formulated as follows:
\begin{equation}
    \label{eq:lr_LapOT}
    \min_{(U, V, g) \in \cC(a, b, r)} q(U \mathrm{diag}(1_r / g) V^\top) - \lambda H(U) - \lambda H(V) - \lambda H(g),
\end{equation}
where $q \colon \R_+^{n \times m} \to \R$ is the quadratic objective function of LapOT defined as
\begin{equation*}
    q(\pi) = \langle \pi, C \rangle + \lambda_x \langle \pi, L_X \pi \rangle + \lambda_y \langle \pi, \pi L_Y \rangle.
\end{equation*}
As in \cite{scetbon2021low}, we can efficiently solve \eqref{eq:lr_LapOT} by utilizing Dykstra's algorithm as a subroutine of a mirror descent algorithm with KL divergence. We defer the details of the optimization scheme to Section \ref{sec:lr_LapOT_optimization} in the appendix.

\section{Theory}
\label{sec:theory}
This section studies the theoretical properties of Laplacian optimal transport. As noted earlier, the Laplacian regularization terms in \eqref{eq:laplacian_terms} encourage the optimal coupling to be smooth with respect to the graph structures of $L_X$ and $L_Y$. Therefore, if $L_X$ and $L_Y$ are induced by graphs with multiple connected components, then the optimal coupling is expected to be block-constant with respect to the partitions induced by the connected components, which leads to a low-rank structure of the optimal coupling. 

Our main result, Theorem \ref{thm:optimal_coupling_projection_bound}, provides a non-asymptotic bound on the deviation of the optimal coupling from its block-averaged version with respect to the partitions induced by the connected components of $L_X$ and $L_Y$. This result shows that the optimal coupling is close to a block-constant matrix when (1) the cost matrix is close to a block-constant matrix, (2) the regularization parameters $\lambda_x$ and $\lambda_y$ are large, or (3) the spectral gaps of $L_X$ and $L_Y$ are large.

\begin{theorem}
    \label{thm:optimal_coupling_projection_bound}
    Let $\pi^\star$ be the solution to $\mathsf{LapOT}_{\lambda_x, \lambda_y, \lambda}(a, b, C, L_X, L_Y)$ with $\lambda_x, \lambda_y > 0$. Assume the following.
    \begin{itemize}
        \item[(i)] $L_X$ and $L_Y$ are induced by graphs with $r$ and $s$ connected components, where $\{A_\alpha\}_{\alpha = 1}^{r}$ and $\{B_\beta\}_{\beta = 1}^{s}$ are the corresponding partitions of the index sets $[n]$ and $[m]$, respectively.
        \item[(ii)] $a$ and $b$ are block-constant with respect to the partitions $\{A_\alpha\}_{\alpha = 1}^{r}$ and $\{B_\beta\}_{\beta = 1}^{s}$, respectively.
    \end{itemize}
    Let $L_X = \Phi_X \Lambda_X \Phi_X^\top$ and $L_Y = \Phi_Y \Lambda_Y \Phi_Y^\top$ be their orthogonal decompositions, where $\Lambda_X$ and $\Lambda_Y$ are diagonal matrices consisting of the eigenvalues $\mu_1^X \le \cdots \le \mu_n^X$ and $\mu_1^Y \le \cdots \le \mu_m^Y$ of $L_X$ and $L_Y$, respectively. For any $\ell \in [n]$ and $h \in [m]$, define the projection matrices as $P_{X, \ell} \triangleq \Phi_{X, \ell} \Phi_{X, \ell}^\top$ and $P_{Y, h} \triangleq \Phi_{Y, h} \Phi_{Y, h}^\top$, where $\Phi_{X, \ell}$ and $\Phi_{Y, h}$ are the submatrices consisting of the first $\ell$ and $h$ columns of $\Phi_X$ and $\Phi_Y$, respectively. Then, $P_{X, r} \pi^\star$ and $\pi^\star P_{Y, s}$ are obtained by block-averaging the rows and columns of $\pi^\star$ with respect to the partitions $\{A_\alpha\}_{\alpha = 1}^{r}$ and $\{B_\beta\}_{\beta = 1}^{s}$, respectively. Moreover, we have
    \begin{equation}
        \label{eq:optimal_coupling_projection_bound}
        \|\pi^\star - P_{X, r} \pi^\star\|_{\mathrm{F}}^2 \le \frac{\|P_{X, r} C - C\|_\infty}{\lambda_x \mu_{r + 1}^X} \quad \text{and} \quad \|\pi^\star - \pi^\star P_{Y, s}\|_{\mathrm{F}}^2 \le \frac{\|C P_{Y, s} - C\|_\infty}{\lambda_y \mu_{s + 1}^Y}.
    \end{equation}
\end{theorem}
\begin{proof}
    From spectral graph theory, e.g., Proposition 2 of \cite{von2007tutorial}, notice that (i) implies that the eigenspaces of $L_X$ and $L_Y$ corresponding to the zero eigenvalues ($\mu_1^X = \cdots = \mu_r^X = 0$ and $\mu_1^Y = \cdots = \mu_s^Y = 0$) are spanned by the indicator vectors of the connected components of the graphs. Accordingly, we have $L_X \Phi_{X, r} = 0$ and $L_Y \Phi_{Y, s} = 0$. Also, we deduce that left multiplication by $P_{X, r}$ and right multiplication by $P_{Y, s}$ are the operators that average matrices with respect to the partitions $\{A_\alpha\}_{\alpha = 1}^{r}$ and $\{B_\beta\}_{\beta = 1}^{s}$, respectively. 

    Now, we compare the objective values of $\pi^\star$ and $P_{X, r} \pi^\star$ in $\mathsf{LapOT}_{\lambda_x, \lambda_y, \lambda}(a, b, C, L_X, L_Y)$. Notice that $P_{X, r} \pi^\star \in \Pi_{a, b}$. To see this, observe that (ii) implies $P_{X, r} a = a$ and $P_{Y, s} b = b$. Hence,
    \begin{equation*}
        P_{X, r} \pi^\star 1_m = P_{X, r} a = a \quad \text{and} \quad (P_{X, r} \pi^\star)^\top 1_n = (\pi^\star)^\top P_{X, r} 1_n = (\pi^\star)^\top 1_n = b.
    \end{equation*}
    As $L_X P_{X, r} = 0$, we have $\langle P_{X, r} \pi^\star, L_X P_{X, r} \pi^\star \rangle = 0$. By the optimality of $\pi^\star$, we have
    \begin{equation*}
        \langle \pi^\star, C \rangle + \lambda_x \langle \pi^\star, L_X \pi^\star \rangle + \lambda_y \langle \pi^\star, \pi^\star L_Y \rangle - \lambda H(\pi^\star) \le \langle P_{X, r} \pi^\star, C \rangle + \lambda_y \langle P_{X, r} \pi^\star, P_{X, r} \pi^\star L_Y \rangle - \lambda H(P_{X, r} \pi^\star).
    \end{equation*}
    As $\langle P_{X, r} \pi^\star, C \rangle = \langle \pi^\star, P_{X, r} C \rangle$, we have
    \begin{equation*}
        \lambda (H(P_{X, r} \pi^\star) - H(\pi^\star)) + \lambda_x \langle \pi^\star, L_X \pi^\star \rangle
        \le
        \langle \pi^\star, P_{X, r} C - C \rangle + \lambda_y \langle P_{X, r} \pi^\star, P_{X, r} \pi^\star L_Y \rangle - \lambda_y \langle \pi^\star, \pi^\star L_Y \rangle.
    \end{equation*}
    Since $P_{X, r} \pi^\star$ is obtained by block-averaging the rows of $\pi^\star$ with respect to the partition $\{A_\alpha\}_{\alpha = 1}^{r}$, we have $H(P_{X, r} \pi^\star) \ge H(\pi^\star)$ due to the concavity of $H$. 
    Also, we observe $\langle P_{X, r} \pi^\star, P_{X, r} \pi^\star L_Y \rangle \le \langle \pi^\star, \pi^\star L_Y \rangle$. To see this, let $\phi_i$ be the $i$-th column of $\Phi_X$. Then,
    \begin{equation*}
        \begin{split}
            \langle \pi^\star, \pi^\star L_Y \rangle - \langle P_{X, r} \pi^\star, P_{X, r} \pi^\star L_Y \rangle 
            & = \mathrm{tr}(\pi^\star L_Y (\pi^\star)^\top) - \mathrm{tr}(\Phi_{X, r}^\top \pi^\star L_Y (\pi^\star)^\top \Phi_{X, r}) \\
            & = \sum_{i = r + 1}^{n} \phi_i^\top \pi^\star L_Y (\pi^\star)^\top \phi_i \\
            & \ge 0,
        \end{split}
    \end{equation*}
    where the inequality holds because $\pi^\star L_Y (\pi^\star)^\top$ is positive semi-definite. Meanwhile, 
    \begin{equation*}
        \langle \pi^\star, P_{X, r} C - C \rangle \le \|\pi^\star\|_1 \|P_{X, r} C - C\|_\infty = \|P_{X, r} C - C\|_\infty.
    \end{equation*}
    Hence, 
    \begin{equation}
        \label{eq:bound_temporary}
        \lambda_x \langle \pi^\star, L_X \pi^\star \rangle
        \le \|P_{X, r} C - C\|_\infty.
    \end{equation}

    Now, we derive a lower bound on $\lambda_x \langle \pi^\star, L_X \pi^\star \rangle$. Let $\tilde{\pi} = \Phi_X^\top \pi^\star$ be the representation of $\pi^\star$ in the eigenbasis of $L_X$. Then, we have
    \begin{equation*}
        \langle \pi^\star, L_X \pi^\star \rangle = \mathrm{tr}((\pi^\star)^\top L_X \pi^\star) = \mathrm{tr}(\tilde{\pi}^\top \Lambda_X \tilde{\pi}) = \sum_{i = 1}^{n} \sum_{j = 1}^{m} \mu_i^X \tilde{\pi}_{i j}^2 \ge \mu_{r + 1}^X \sum_{i = r + 1}^{n} \sum_{j = 1}^{m} \tilde{\pi}_{i j}^2.
    \end{equation*}
    Meanwhile,
    \begin{equation*}
        P_{X, r} \pi^\star
        = \Phi_{X, r} \Phi_{X, r}^\top \pi^\star
        = \Phi_X \begin{bmatrix}
            \Phi_{X, r}^\top \\ 0
        \end{bmatrix} \pi^\star
        = \Phi_X \begin{bmatrix}
            \tilde{\pi}_{1:r} \\
            0
        \end{bmatrix},
    \end{equation*}
    where $\tilde{\pi}_{1:r}$ is the submatrix of $\tilde{\pi}$ containing the first $r$ rows. Hence, from $\pi^\star = \Phi_X \tilde{\pi}$, we have
    \begin{equation*}
        \|\pi^\star - P_{X, r} \pi^\star\|_{\mathrm{F}}^2 = \sum_{i = r + 1}^{n} \sum_{j = 1}^{m} \tilde{\pi}_{i j}^2.
    \end{equation*}
    Hence, $\mu_{r + 1}^X \|\pi^\star - P_{X, r} \pi^\star\|_{\mathrm{F}}^2 \le \langle \pi^\star, L_X \pi^\star \rangle$, which, together with \eqref{eq:bound_temporary}, implies the first inequality in \eqref{eq:optimal_coupling_projection_bound}. The second inequality in \eqref{eq:optimal_coupling_projection_bound} can be proved similarly by comparing the objective values of $\pi^\star$ and $\pi^\star P_{Y, s}$.
\end{proof}

From the bounds in \eqref{eq:optimal_coupling_projection_bound}, we immediately deduce that $\pi^\star = P_{X, r} \pi^\star$ and $\pi^\star = \pi^\star P_{Y, s}$ if $P_{X, r} C = C$ and $C P_{Y, s} = C$, respectively. In other words, if the matching cost is determined at the cluster level, then the optimal coupling indeed captures the cluster-level matching due to the Laplacian regularization terms and thus admits a low-rank structure with non-negative rank, defined in \eqref{eq:non_negative_rank}, bounded by the number of clusters. This is summarized in the following corollary.

\begin{corollary}
    In Theorem \ref{thm:optimal_coupling_projection_bound}, we deduce the following.
    \begin{itemize}
        \item[(i)] If $C$ is row block-constant with respect to the partition $\{A_\alpha\}_{\alpha = 1}^{r}$, namely, $P_{X, r} C = C$, then $\pi^\star$ is row block-constant with respect to the partition $\{A_\alpha\}_{\alpha = 1}^{r}$, namely, $\pi^\star = P_{X, r} \pi^\star$, and $\mathrm{rk}_+(\pi^\star) \le r$.
        \item[(ii)] If $C$ is column block-constant with respect to the partition $\{B_\beta\}_{\beta = 1}^{s}$, namely, $C P_{Y, s} = C$, then $\pi^\star$ is column block-constant with respect to the partition $\{B_\beta\}_{\beta = 1}^{s}$, namely, $\pi^\star = \pi^\star P_{Y, s}$, and $\mathrm{rk}_+(\pi^\star) \le s$.
        \item[(iii)] If both (i) and (ii) hold, or equivalently, $P_{X, r} C P_{Y, s} = C$, then $\pi^\star$ is block-constant with respect to the product partition $\{A_\alpha \times B_\beta\}_{\alpha \in [r], \beta \in [s]}$, namely, $\pi^\star = P_{X, r} \pi^\star P_{Y, s}$, and $\mathrm{rk}_+(\pi^\star) \le \min\{r, s\}$.
    \end{itemize}
\end{corollary}

Meanwhile, from the bounds in \eqref{eq:optimal_coupling_projection_bound}, we deduce that $\pi^\star$ converges to its block-averaged version as $\lambda_x$ and $\lambda_y$ increase, showing that stronger Laplacian regularization leads to a more pronounced low-rank structure of the optimal coupling. Indeed, at the limit of $\lambda_x \to \infty$ or $\lambda_y \to \infty$, the optimal coupling converges to its block-averaged version. In this case, we can further characterize the limit of the optimal coupling as follows.

\begin{proposition}
    In Theorem \ref{thm:optimal_coupling_projection_bound}, we deduce the following.
    \begin{itemize}
        \item[(i)] If $\lambda_x \to \infty$, then $\pi^\star$ converges as follows:
        \begin{equation*}
            \lim_{\lambda_x \to \infty} \pi^\star = \argmin_{\pi \in \Pi_{a, b}} \langle \pi, P_{X, r} C \rangle + \lambda_y \langle \pi, \pi L_Y \rangle - \lambda H(\pi).
        \end{equation*}
        \item[(ii)] If $\lambda_y \to \infty$, then $\pi^\star$ converges as follows:
        \begin{equation*}
            \lim_{\lambda_y \to \infty} \pi^\star = \argmin_{\pi \in \Pi_{a, b}} \langle \pi, C P_{Y, s} \rangle + \lambda_x \langle \pi, L_X \pi \rangle - \lambda H(\pi).
        \end{equation*}
        \item[(iii)] If $\lambda_x, \lambda_y \to \infty$, then $\pi^\star$ converges as follows:
        \begin{equation*}
            \lim_{\lambda_x, \lambda_y \to \infty} \pi^\star = \argmin_{\pi \in \Pi_{a, b}} \langle \pi, P_{X, r} C P_{Y, s} \rangle - \lambda H(\pi).
        \end{equation*}
    \end{itemize}
\end{proposition}
\begin{proof}
    To show (i), let $\pi^\star_{\lambda_x}$ denote the optimal coupling for a given $\lambda_x$. Now, consider a sequence $\{\lambda_x^{(k)}\}_{k \in \mathbb{N}}$ such that $\lambda_x^{(k)} \to \infty$ as $k \to \infty$. Since $\Pi_{a, b}$ is compact, there exists a convergent subsequence of $\{\pi^\star_{\lambda_x^{(k)}}\}_{k \in \mathbb{N}}$; by swapping the original sequence with its convergent subsequence, we may assume that $\{\pi^\star_{\lambda_x^{(k)}}\}_{k \in \mathbb{N}}$ converges to some $\tilde{\pi} \in \Pi_{a, b}$. We now claim 
    \begin{equation*}
        \tilde{\pi} = \argmin_{\pi \in \Pi_{a, b}} \langle \pi, P_{X, r} C \rangle + \lambda_y \langle \pi, \pi L_Y \rangle - \lambda H(\pi) =: \pi^\star_\infty.
    \end{equation*}
    We first note that $\pi^\star_\infty = P_{X, r} \pi^\star_\infty$. To see this, as in the proof of Theorem \ref{thm:optimal_coupling_projection_bound}, observe that $P_{X, r} \pi^\star_\infty \in \Pi_{a, b}$, $H(P_{X, r} \pi^\star_\infty) \ge H(\pi^\star_\infty)$, and $\langle P_{X, r} \pi^\star_\infty, P_{X, r} \pi^\star_\infty L_Y \rangle \le \langle \pi^\star_\infty, \pi^\star_\infty L_Y \rangle$. Hence,
    \begin{equation*}
        \langle P_{X, r} \pi^\star_\infty, P_{X, r} C \rangle + \lambda_y \langle P_{X, r} \pi^\star_\infty, P_{X, r} \pi^\star_\infty L_Y \rangle - \lambda H(P_{X, r} \pi^\star_\infty) \le \langle \pi^\star_\infty, P_{X, r} C \rangle + \lambda_y \langle \pi^\star_\infty, \pi^\star_\infty L_Y \rangle - \lambda H(\pi^\star_\infty),
    \end{equation*}
    where we also use $\langle P_{X, r} \pi^\star_\infty, P_{X, r} C \rangle = \langle \pi^\star_\infty, P_{X, r} C \rangle$. By the uniqueness of the minimizer $\pi^\star_\infty$, we conclude that $\pi^\star_\infty = P_{X, r} \pi^\star_\infty$.

    For notational convenience, let $\pi^\star_{\lambda_x^{(k)}} = \pi^\star_k$ for $k \in \N$. Then, by the optimality of $\pi^\star_k$, we have
    \begin{equation}
        \label{eq:optimality_limit}
        \langle \pi^\star_k, C \rangle + \lambda_x^{(k)} \langle \pi^\star_k, L_X \pi^\star_k \rangle + \lambda_y \langle \pi^\star_k, \pi^\star_k L_Y \rangle - \lambda H(\pi^\star_k) \le \langle \pi^\star_\infty, C \rangle + \lambda_y \langle \pi^\star_\infty, \pi^\star_\infty L_Y \rangle - \lambda H(\pi^\star_\infty),
    \end{equation}
    where we use $\langle \pi^\star_\infty, L_X \pi^\star_\infty \rangle = 0$ from $\pi^\star_\infty = P_{X, r} \pi^\star_\infty$. Now, we claim 
    \begin{equation}
        \label{eq:limit_objective}
        \lim_{k \to \infty} \left(\langle \pi^\star_k, C \rangle + \lambda_x^{(k)} \langle \pi^\star_k, L_X \pi^\star_k \rangle + \lambda_y \langle \pi^\star_k, \pi^\star_k L_Y \rangle - \lambda H(\pi^\star_k)\right) = \langle \tilde{\pi}, P_{X, r} C \rangle + \lambda_y \langle \tilde{\pi}, \tilde{\pi} L_Y \rangle - \lambda H(\tilde{\pi}).
    \end{equation}
    To see this, use \eqref{eq:optimal_coupling_projection_bound} to deduce that 
    \begin{equation*}
        \lim_{k \to \infty} P_{X, r} \pi^\star_k = \lim_{k \to \infty} \pi^\star_k = \tilde{\pi}.
    \end{equation*}
    Hence, 
    \begin{equation}
        \label{eq:limit_temp1}
        \lim_{k \to \infty} \langle \pi^\star_k, C \rangle = \lim_{k \to \infty} \langle P_{X, r} \pi^\star_k, C \rangle = \lim_{k \to \infty} \langle \pi^\star_k, P_{X, r} C \rangle = \langle \tilde{\pi}, P_{X, r} C \rangle
    \end{equation}
    and
    \begin{equation}
        \label{eq:limit_temp2}
        \lim_{k \to \infty} \left(\lambda_y \langle \pi^\star_k, \pi^\star_k L_Y \rangle - \lambda H(\pi^\star_k)\right) = \lambda_y \langle \tilde{\pi}, \tilde{\pi} L_Y \rangle - \lambda H(\tilde{\pi}).
    \end{equation}
    Meanwhile, as in the proof of Theorem \ref{thm:optimal_coupling_projection_bound}, we have the following from the optimality of $\pi^\star_k$:
    \begin{equation*}
        \lambda_x^{(k)} \langle \pi^\star_k, L_X \pi^\star_k \rangle \le \langle P_{X, r} \pi^\star_k, C \rangle + \lambda_y \langle P_{X, r} \pi^\star_k, P_{X, r} \pi^\star_k L_Y \rangle - \lambda H(P_{X, r} \pi^\star_k) - \langle \pi^\star_k, C \rangle - \lambda_y \langle \pi^\star_k, \pi^\star_k L_Y \rangle + \lambda H(\pi^\star_k),
    \end{equation*}
    where the right-hand side converges to $0$ as $k \to \infty$ by the previous two limits. Hence, 
    \begin{equation}
        \label{eq:limit_temp3}
        \lim_{k \to \infty} \lambda_x^{(k)} \langle \pi^\star_k, L_X \pi^\star_k \rangle = 0.
    \end{equation}
    Combining \eqref{eq:limit_temp1}, \eqref{eq:limit_temp2}, and \eqref{eq:limit_temp3}, we obtain \eqref{eq:limit_objective}. Then, combining \eqref{eq:optimality_limit} and \eqref{eq:limit_objective}, we have
    \begin{equation*}
        \langle \tilde{\pi}, P_{X, r} C \rangle + \lambda_y \langle \tilde{\pi}, \tilde{\pi} L_Y \rangle - \lambda H(\tilde{\pi}) \le \langle \pi^\star_\infty, P_{X, r} C \rangle + \lambda_y \langle \pi^\star_\infty, \pi^\star_\infty L_Y \rangle - \lambda H(\pi^\star_\infty),
    \end{equation*}
    where we use $\langle \pi^\star_\infty, C \rangle = \langle \pi^\star_\infty, P_{X, r} C \rangle$ from $\pi^\star_\infty = P_{X, r} \pi^\star_\infty$. By the uniqueness of the minimizer $\pi^\star_\infty$, we conclude that $\tilde{\pi} = \pi^\star_\infty$.

    Hence, we have shown that for any sequence $\{\lambda_x^{(k)}\}_{k \in \mathbb{N}}$ such that $\lambda_x^{(k)} \to \infty$, the corresponding sequence of optimal couplings $\{\pi^\star_{\lambda_x^{(k)}}\}_{k \in \mathbb{N}}$ has a convergent subsequence that converges to $\pi^\star_\infty$. As the limit is independent of the choice of the sequence $\{\lambda_x^{(k)}\}_{k \in \mathbb{N}}$, we conclude that $\pi^\star_{\lambda_x} \to \pi^\star_\infty$ as $\lambda_x \to \infty$. This proves (i). The proofs of (ii) and (iii) follow similarly.
\end{proof}

The above results describe an idealized setting in which clusters are represented as connected components of the similarity graphs. In the numerical experiments below, however, the similarity graphs are constructed from RBF kernels and are typically connected, while the degree-based marginals need not be block-constant. In this setting, the appropriate analog of block-constancy is low-frequency smoothness with respect to the graph Laplacians. The following result makes this connection precise: for arbitrary graph Laplacians, the LapOT solution is close to its projection onto the low-frequency eigenspaces of the two graphs, with an error controlled by the regularization parameters and the corresponding spectral gaps. This provides a theoretical explanation for why the coupling obtained by LapOT can be used in RSC to reveal approximate cluster structure even when the input graphs are
connected.

\begin{proposition} \label{prop:proj_second_bound}
    Let $\pi^\star$ be the solution of $\mathsf{LapOT}_{\lambda_x,\lambda_y,\lambda}(a,b,C,L_X,L_Y)$, where $L_X$ and $L_Y$ are arbitrary graph Laplacians having the eigenvalues $\mu_1^X \le \cdots \le \mu_n^X$ and $\mu_1^Y \le \cdots \le \mu_m^Y$ of $L_X$ and $L_Y$, respectively. Let $P_{X, \ell}$ and $P_{Y, h}$ denote the orthogonal projections onto the first $\ell$ and $h$ eigenvectors of $L_X$ and $L_Y$, respectively, as in Theorem \ref{thm:optimal_coupling_projection_bound}. Define
    \[
        \tau_\lambda(C)
        =
        \min_{\pi\in\Pi_{a,b}}
        \{\langle \pi,C\rangle-\lambda H(\pi)\}.
    \]
    Then, for any $\ell < n$ and $h < m$, we have
    \[
        \|\pi^\star - P_{X, \ell} \pi^\star\|_{\mathrm{F}}^2
        \le
        \frac{F(\pi^\star) - \tau_\lambda(C)}{\lambda_x \mu^X_{\ell + 1}}
        \quad 
        \text{and}
        \quad
        \|\pi^\star-\pi^\star P_{Y,h}\|_{\mathrm{F}}^2
        \le
        \frac{F(\pi^\star) - \tau_\lambda(C)}{\lambda_y \mu^Y_{h + 1}},
    \]
    where $F(\pi) = \langle \pi, C \rangle + \lambda_x \langle \pi, L_X \pi \rangle + \lambda_y \langle \pi, \pi L_Y \rangle - \lambda H(\pi)$.
\end{proposition}
\begin{proof}
    By optimality, $F(\pi^\star) \le F(\pi)$ for any $\pi \in \Pi_{a, b}$. Since $\langle \pi^\star, C \rangle - \lambda H(\pi^\star)\ge \tau_\lambda(C)$, we obtain
    \[
        \lambda_x \langle \pi^\star, L_X \pi^\star \rangle
        +
        \lambda_y \langle \pi^\star, \pi^\star L_Y \rangle
        \le
        F(\pi) - \tau_\lambda(C).
    \]
    As this is true for any $\pi \in \Pi_{a, b}$, we have 
    \begin{equation*}
        \lambda_x \langle \pi^\star, L_X \pi^\star \rangle
        +
        \lambda_y \langle \pi^\star, \pi^\star L_Y \rangle
        \le
        F(\pi^\star) - \tau_\lambda(C).
    \end{equation*}
    Now, recall the orthogonal decomposition $L_X = \Phi_X \Lambda_X \Phi_X^\top$ in Theorem \ref{thm:optimal_coupling_projection_bound}. Then,
    \[
        \langle \pi^\star, L_X \pi^\star \rangle
        =
        \sum_{i = 1}^{n} \mu_i^X \|e_i^\top \Phi_X^\top \pi^\star\|_2^2
        \ge
        \mu_{\ell + 1}^X \sum_{i = \ell + 1}^{n} \|e_i^\top \Phi_X^\top \pi^\star\|_2^2
        =
        \mu_{\ell + 1}^X
        \|\pi^\star-P_{X, \ell} \pi^\star\|_{\mathrm{F}}^2.
    \]
    Hence, the first bound follows, and the second bound follows analogously.
\end{proof}

Another conclusion from Proposition~\ref{prop:proj_second_bound} is
\[
    \|\pi^\star - P_{X, \ell} \pi^\star P_{Y, h}\|_{\mathrm{F}}
    \le
    \sqrt{\frac{F(\pi^\star) - \tau_\lambda(C)}{\lambda_x \mu^X_{\ell + 1}}}
    +
    \sqrt{\frac{F(\pi^\star) - \tau_\lambda(C)}{\lambda_y \mu^Y_{h + 1}}},
\]
which follows from
\[
    \pi^\star-P_{X,\ell}\pi^\star P_{Y,h}
    =
    (I-P_{X,\ell})\pi^\star
    +
    P_{X,\ell}\pi^\star(I-P_{Y,h})
\]
and the triangle inequality. 

The quantity $F(\pi^\star) - \tau_\lambda(C) \ge 0$ is the gap between the LapOT and entropic-OT optimal values, both computed by the solver. Hence, Proposition \ref{prop:proj_second_bound} thus provides an \emph{a posteriori} certificate: after solving LapOT, one may evaluate the right-hand side directly to verify that $\pi^\star$ concentrates on the low-frequency eigenspaces of $L_X$ and $L_Y$. The bound is non-vacuous whenever this value gap is small relative to $\lambda_x \mu^X_{\ell + 1}$ and $\lambda_y \mu^Y_{h + 1}$. Accordingly, Proposition \ref{prop:proj_second_bound} states that when the similarity graphs have a pronounced eigengap after $\ell$ and $h$ eigenvectors, the optimal coupling is close to a matrix whose rows and columns vary primarily along the corresponding approximate cluster indicators. This provides the theoretical mechanism behind the use of LapOT inside RSC, even when the graphs are connected and the marginals are not block-constant.

To conclude, Theorem~\ref{thm:optimal_coupling_projection_bound} should be viewed as the exact disconnected-graph limit, while Proposition~\ref{prop:proj_second_bound} explains the connected-graph regime used in the experiments in the following sections. In Section \ref{subsec:stockmarket}, we indeed report the upper bound on the right-hand side of Proposition~\ref{prop:proj_second_bound} for the stock market data, numerically confirming the decay of the projection error as the dimension of the low-frequency eigenspaces increases.

\section{Refined Simultaneous Clustering}
\label{sec:rsc}
This section introduces a new approach for clustering two point clouds simultaneously, which we call Refined Simultaneous Clustering (RSC). We have seen in the previous sections that the optimal coupling of LapOT can capture the cluster structure of the two point clouds. RSC is a practical method that distills the cluster structure from the optimal coupling of LapOT and uses it to perform simultaneous clustering of the two point clouds with the goal of making the clustering of both clouds more aligned with each other. 

\begin{algorithm}[!htbp]
    \setstretch{1.25}
    \caption{Refined Simultaneous Clustering}
    \label{alg:RSC}
    \begin{algorithmic}[1]
        \Input matching cost $C \in \mathbb{R}^{n \times m}$, marginal weights $a \in \Delta_n$ and $b \in \Delta_m$.
        \Input similarity matrices $K_X \in \mathbb{R}^{n \times n}$ and $K_Y \in \mathbb{R}^{m \times m}$.
        \Input hyperparameters $\lambda_x, \lambda_y, \lambda$ for regularization and $k, k' \in \N$ for clustering.        
        \State Define $L_X = \mathrm{diag}(K_X 1_n) - K_X$ and $L_Y = \mathrm{diag}(K_Y 1_m) - K_Y$.
        \State Find the solution $\pi^\star$ of $\mathsf{LapOT}_{\lambda_x, \lambda_y, \lambda}(a, b, C, L_X, L_Y)$.
        \State Apply k-means with $k'$ clusters to the rows, columns of $\pi^\star$ to define $\pi_{\mathsf{switch}}^X$, $\pi_{\mathsf{switch}}^Y$ per \eqref{eq:switch_X}, \eqref{eq:switch_Y}.
        \State $\tilde{K}_X \gets K_X \odot \pi_{\mathsf{switch}}^X$ and $\tilde{K}_Y \gets K_Y \odot \pi_{\mathsf{switch}}^Y$
        \State Update the graph Laplacians: $\tilde{L}_X \gets \mathrm{diag}(\tilde{K}_X 1_n) - \tilde{K}_X$ and $\tilde{L}_Y \gets \mathrm{diag}(\tilde{K}_Y 1_m) - \tilde{K}_Y$
        \State Apply spectral clustering with $k$ clusters to $\tilde{L}_X$ and $\tilde{L}_Y$ to obtain the final clusters of $X$ and $Y$.
        \Output Final clusters of $X$ and $Y$.
    \end{algorithmic}
\end{algorithm}

The full procedure of RSC is summarized in Algorithm \ref{alg:RSC}. After obtaining the optimal coupling $\pi^\star$ from LapOT, we apply k-means to the rows and columns of $\pi^\star$ and define the switch matrices $\pi^X_{\mathsf{switch}} \in \{0, 1\}^{n \times n}$ and $\pi^Y_{\mathsf{switch}} \in \{0, 1\}^{m \times m}$ as follows:
\begin{align}
    (\pi^X_{\mathsf{switch}})_{i j} & = \begin{cases}
    1 & \text{if} ~ \text{row} ~ i ~ \text{and} ~ \text{row} ~ j ~ \text{belong to the same cluster}, \\
    0 & \text{otherwise},
    \end{cases} \label{eq:switch_X} \\
    (\pi^Y_{\mathsf{switch}})_{ij} & = \begin{cases}
    1 & \text{if} ~ \text{column} ~ i ~ \text{and} ~ \text{column} ~ j ~ \text{belong to the same cluster}, \\
    0 & \text{otherwise}.
    \end{cases} \label{eq:switch_Y}
\end{align}
By applying k-means to the rows (resp.\ columns) of $\pi^\star$, we cluster together points in $X$ (resp.\ $Y$) that have a similar fuzzy matching with $Y$ (resp.\ $X$). Then, in Step 4 of Algorithm \ref{alg:RSC}, we refine the original input similarity matrices $K_X$ and $K_Y$ by multiplying them element-wise with the switch matrices $\pi^X_{\mathsf{switch}}$ and $\pi^Y_{\mathsf{switch}}$, respectively. Then, we recalculate the refined graph Laplacians $\tilde{L}_X$ and $\tilde{L}_Y$ based on the refined similarity matrices $\tilde{K}_X$ and $\tilde{K}_Y$, followed by the usual spectral clustering procedure to obtain the final clusters of $X$ and $Y$. Here, the idea is to condition the final clustering procedure on the similarity matrices refined by the switch matrices, expecting the cluster formation in $X$ and $Y$ to reflect the matching relations between the clouds, and hence, to be more aligned with each other. 

It is worth noting the similarity between RSC and the standard spectral clustering procedure. Spectral clustering is a two-step procedure that first obtains suitable spectral embeddings via various methods like Laplacian eigenmaps \cite{belkin2001laplacian} or diffusion maps \cite{coifman2006diffusion}, and then applies k-means to the spectral embeddings to obtain the final clusters of a single point cloud. In RSC, we first obtain the optimal coupling of LapOT, which encodes the cluster-aware matching information between the two point clouds, allowing us to leverage the shared information between the two point clouds in the subsequent clustering step. 

Figure \ref{fig:RSC_motivating}(b) earlier in the paper shows an example of the RSC method applied to two point clouds $X$ and $Y$ taken from the CAPOD dataset \cite{papadakis2014canonically}. Here, Algorithm \ref{alg:RSC} is applied with $k' = 3$ for the switch matrices in Step 3 and $k = 5$ for the final spectral clustering in Step 6. Unlike the independent clustering of $X$ and $Y$ shown in Figure \ref{fig:RSC_motivating}(a), the RSC method correctly captures the matching cluster structure between the two clouds, resulting in aligned clusters. The input similarity matrices are computed using the RBF kernel, while the marginal weights are based on the degrees from the similarity matrices as explained below. The cost matrix is computed using the Wasserstein-1 distance between the distance profiles of points in $X$ and $Y$ as described in Section \ref{sec:preliminaries}.

\paragraph{Switch Matrices.}
The role of the switch matrices $\pi^X_{\mathsf{switch}}$ and $\pi^Y_{\mathsf{switch}}$ is to capture the coarse cluster structure of the two point clouds based on the optimal coupling $\pi^\star$ before moving on to the final clustering step. Hence, the choice of the number of clusters $k'$ for the switch matrices governs the rough partition of the point clouds, distilling the cluster structure of the match encoded by the optimal coupling of LapOT. In practice, we choose $k'$ to be slightly more than half the number of clusters $k$ for the final clustering step, so that the switch matrices capture the coarse cluster structure of the two clouds, while the subsequent clustering step can still produce clusters with enough granularity. Figure \ref{fig: non-uni cluster by pi switch} visualizes the 3D human shapes in Figure \ref{fig:RSC_motivating} clustered according to the switch matrices $\pi^X_{\mathsf{switch}}$ and $\pi^Y_{\mathsf{switch}}$ with $k' = 3$. We can see that the rough cluster structure of the two clouds is captured by the switch matrices, providing a good starting point for the final clustering step.

\begin{figure}[!htbp]
  \centering
  \includegraphics[width=0.48\linewidth, trim=0.6in 0.6in 0.6in 0.6in, clip]{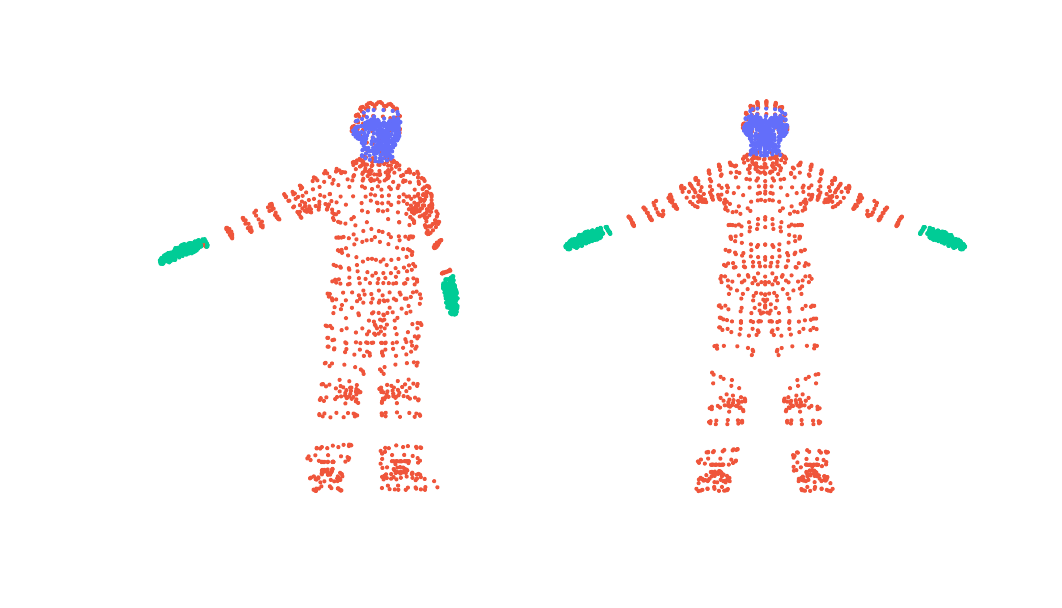}
  \caption{Clustering of two 3D human shapes according to the clustering structure encoded in the switch matrices $\pi^X_{\mathsf{switch}}$ and $\pi^Y_{\mathsf{switch}}$.}
  \label{fig: non-uni cluster by pi switch}
\end{figure}

\paragraph{Degree-Based Marginals.}
While it is common to set $a, b$ to be uniform distributions in practice, we found that using degree-based marginals can improve the performance of RSC. The idea is to assign higher weights to points that are more significant in the similarity graph, which can help to capture the cluster structure of the point clouds more effectively. Precisely, we define the degree-based marginals as follows:
\begin{equation*}
    a_i = \frac{\mathrm{deg}(X_i)}{\sum_{i' = 1}^n \mathrm{deg}(X_{i'})} \quad \text{and} \quad b_j = \frac{\mathrm{deg}(Y_j)}{\sum_{j' = 1}^m \mathrm{deg}(Y_{j'})},
\end{equation*}
where $\mathrm{deg}(X_i) = \sum_{i' = 1}^n (K_X)_{i i'}$ and $\mathrm{deg}(Y_j) = \sum_{j' = 1}^m (K_Y)_{j j'}$ are the degrees of points $X_i$ and $Y_j$ in the similarity graphs, respectively. In this case, the distance profiles to define the cost matrix $C$ as described in Section \ref{sec:preliminaries} are also based on the degree-based marginals.

\paragraph{Discussion.}
Figure \ref{fig:dogs_dolphins} shows additional examples of RSC applied to two point clouds from the CAPOD dataset. The first row shows two 3D dog shapes, while the second row shows two 3D dolphin shapes. In both cases, RSC produces more consistent clusters across the two point clouds compared to independent clustering. On top of these, we also experimented with the low-rank version of LapOT in Step 2 of Algorithm \ref{alg:RSC} with $r = 10$, which also yields consistent clusters across the two point clouds. Overall, these results demonstrate the effectiveness of RSC in capturing the cluster structure of two point clouds simultaneously, leveraging the shared information between them to produce more aligned clusters. Of course, like any clustering method, RSC is not guaranteed to produce perfectly consistent clusters across the two point clouds and depends on the choice of hyperparameters and similarity matrices. We leave a more detailed study of the hyperparameter selection for RSC to future work.

\begin{figure}[!htbp]
  \centering
  \includegraphics[width=0.318\textwidth, trim=0.3in 0.5in 0.3in 0.5in, clip]{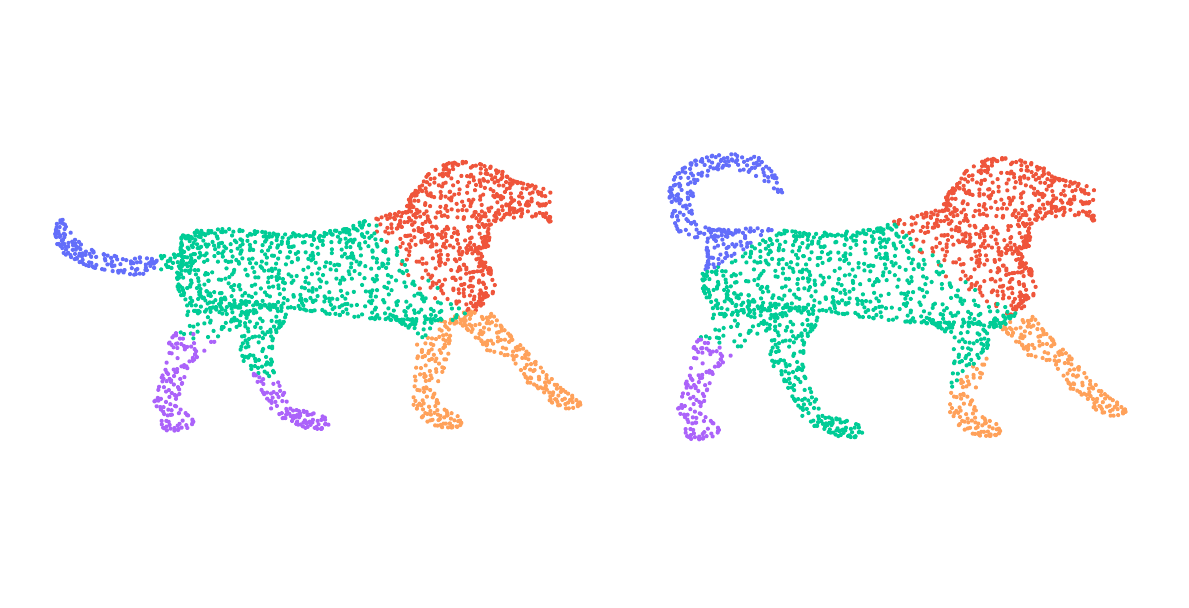}
  \includegraphics[width=0.318\textwidth, trim=0.3in 0.5in 0.3in 0.5in, clip]{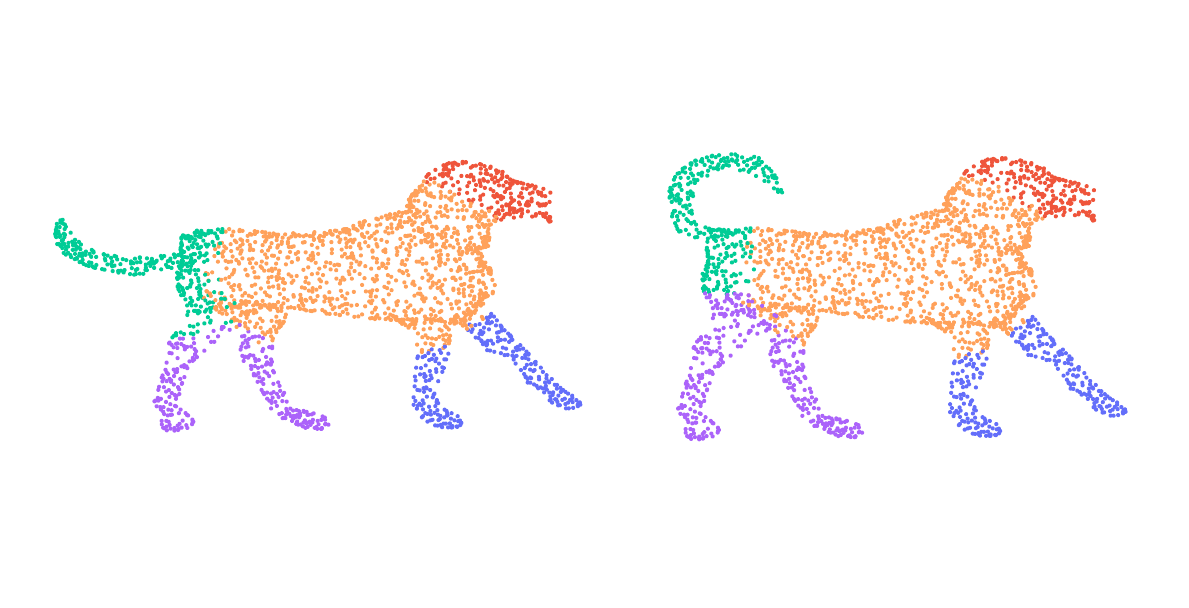}
  \includegraphics[width=0.318\textwidth, trim=0.3in 0.5in 0.3in 0.5in, clip]{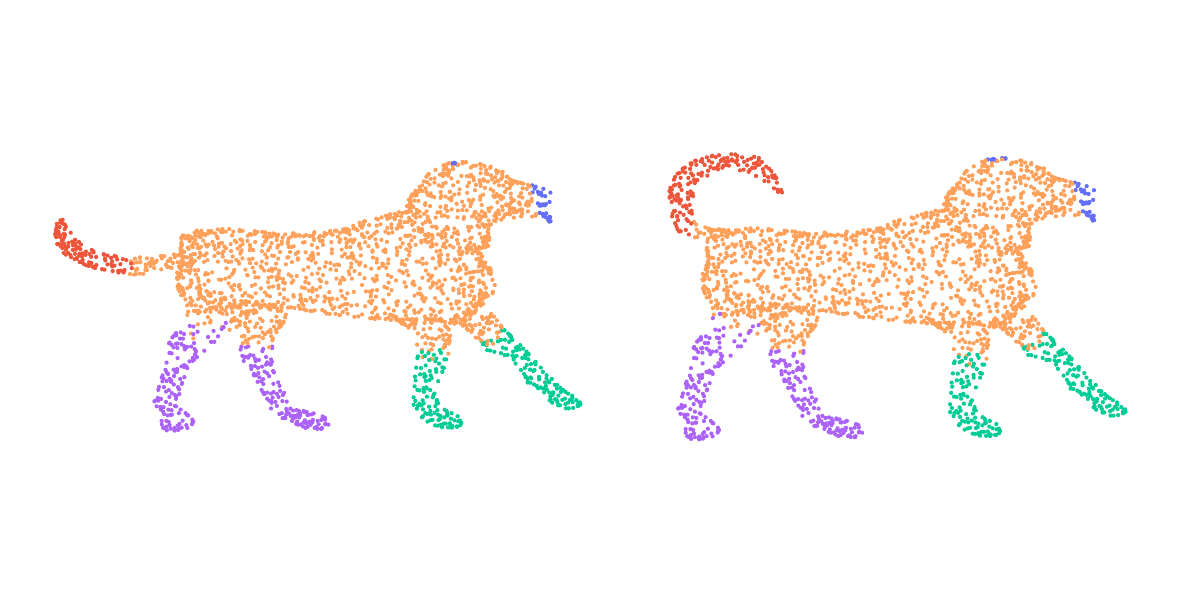}

  \subfloat[Independent Clustering]{
    \includegraphics[width=0.318\textwidth, trim=0.3in 0.5in 0.3in 0.5in, clip]{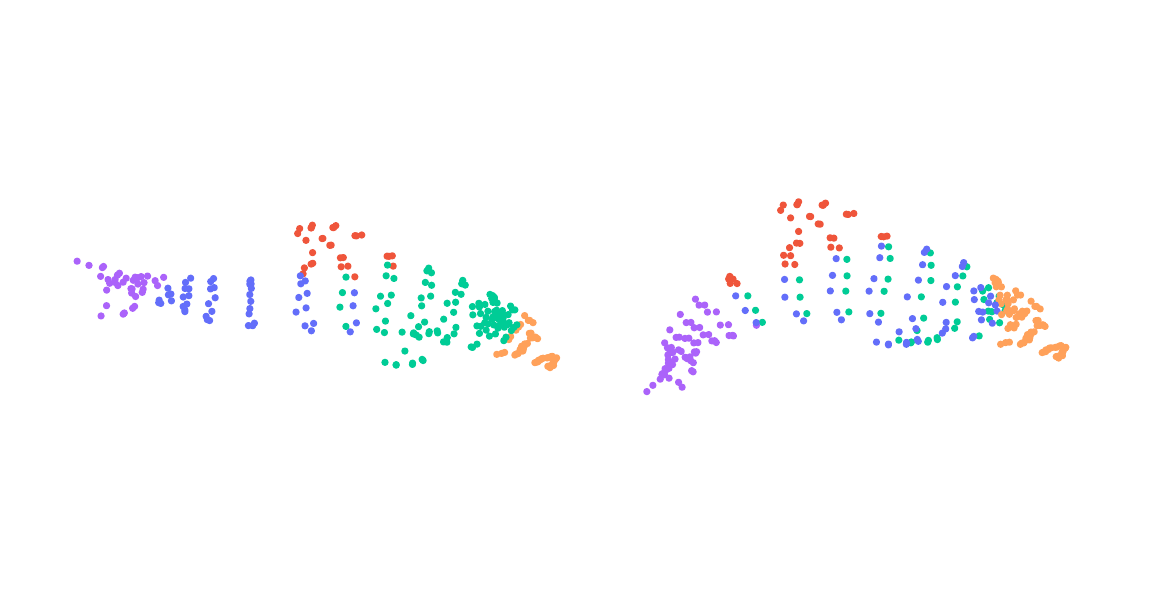}
  }
  \subfloat[RSC]{
    \includegraphics[width=0.318\textwidth, trim=0.3in 0.5in 0.3in 0.5in, clip]{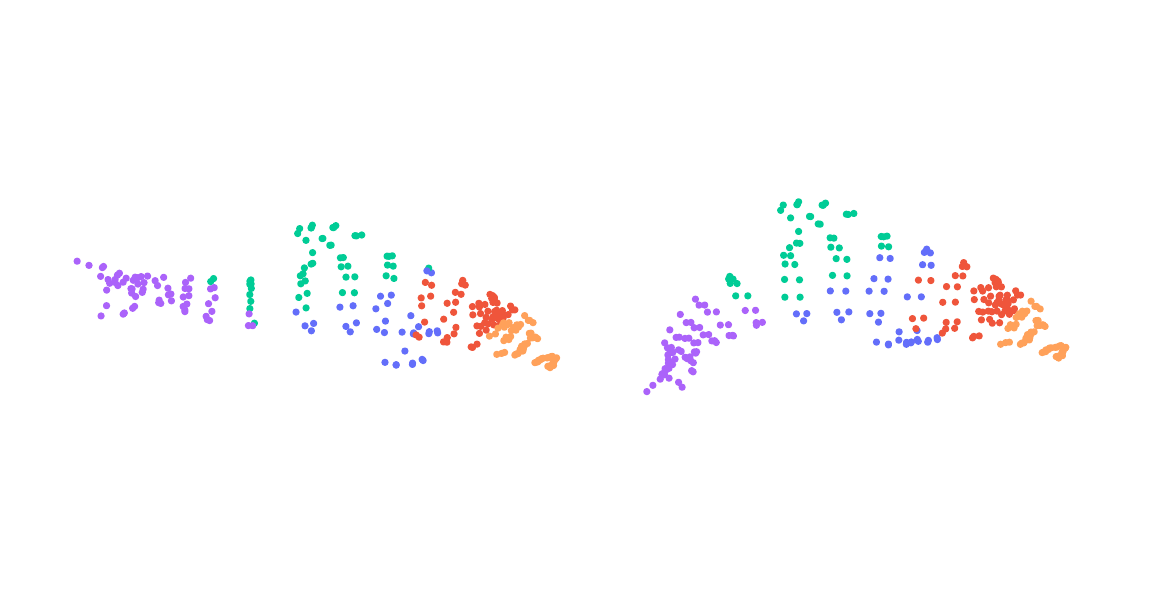}
  }
  \subfloat[RSC (Low-Rank)]{
    \includegraphics[width=0.318\textwidth, trim=0.3in 0.5in 0.3in 0.5in, clip]{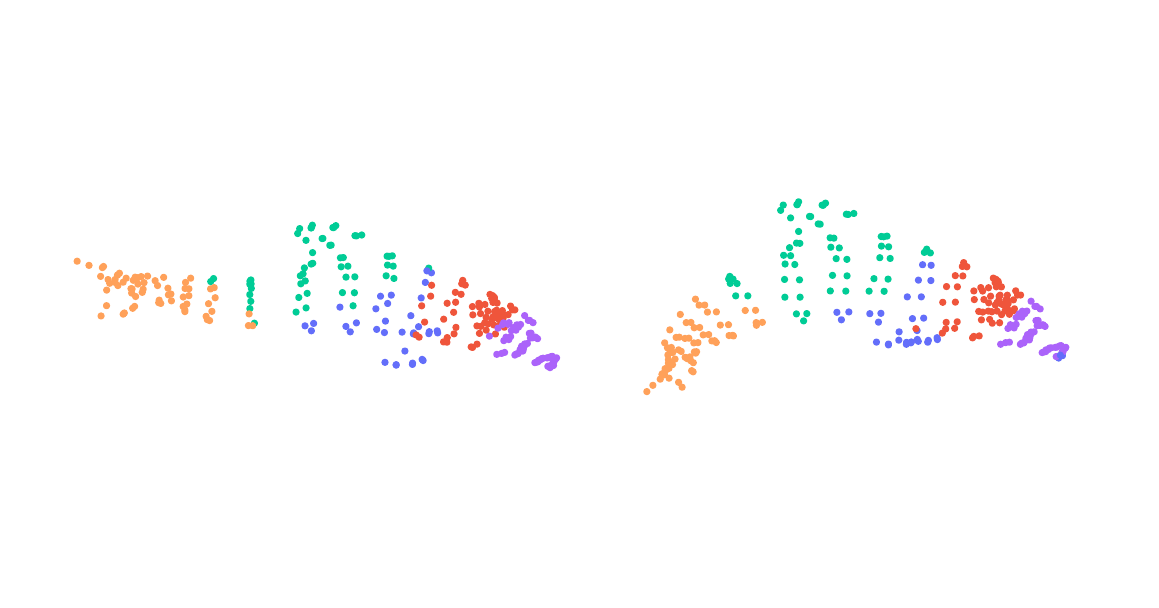}
  }
  \caption{Other examples of RSC applied to two point clouds from the CAPOD dataset. The first row shows two 3D dog shapes, while the second row shows two 3D dolphin shapes. (c) shows the results of RSC with the low-rank version of LapOT in Step 2 of Algorithm \ref{alg:RSC} with $r = 10$.}
  \label{fig:dogs_dolphins}
\end{figure}

\section{Applications}
\label{sec:applications}
\subsection{Alignment via Refined Simultaneous Clustering}
Finding a rigid transformation that best aligns two point clouds is a fundamental problem in many applications, including computer vision, robotics, and 3D modeling. Once we establish a correspondence between the two clouds, we can formulate an orthogonal Procrustes problem to estimate the underlying rigid transformation. In this section, we demonstrate how this can be done by applying RSC. 

We pick a 3D human shape from the CAPOD dataset \cite{papadakis2014canonically} as our point cloud $X$ and generate a second point cloud $Y$ by applying a random rotation to $X$ and adding isotropic Gaussian noise $N(0, \sigma^2 I_3)$, which is shown in Figure \ref{fig:alignment}(a). We begin by applying RSC, following the same hyperparameter selection procedure used to produce Figure \ref{fig:RSC_motivating}(b) as described in Section \ref{sec:rsc}; the result is shown in Figure \ref{fig:alignment}(b). Then, we calculate the centroids of each cluster in $X$ and $Y$ and assign to each centroid a weight proportional to the sum of degrees of the points in its cluster. Now, we consider matching between the centroids of the clusters in $X$ and $Y$ following the usual distance profile matching \eqref{eq:discrete_ot} in Section \ref{sec:preliminaries}. Then, for $k = 5$ matched cluster pairs, we perform the pointwise matching between the points in the matched clusters following the distance profile matching again, which can be done in parallel for each matched cluster pair. Finally, based on the established correspondence between the two clouds, we estimate the underlying rotation by solving the orthogonal Procrustes problem.

\begin{figure}[ht]
    \centering
    \subfloat[Point Clouds]{
        \includegraphics[width=0.48\textwidth, trim=0.6in 0.6in 0.6in 0.6in, clip]{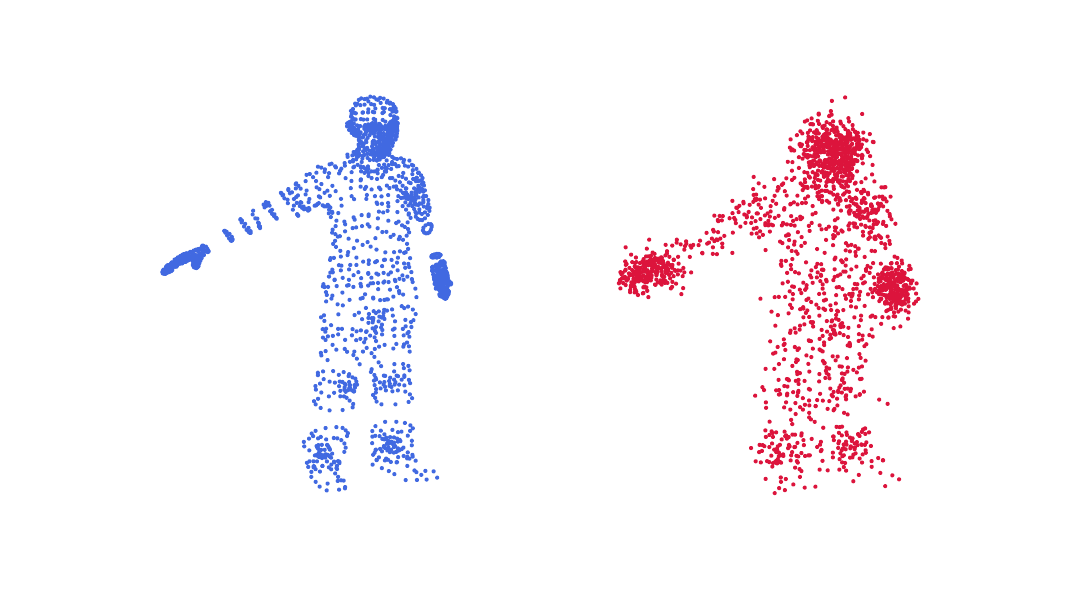}
    }
    \subfloat[RSC]{
        \includegraphics[width=0.48\textwidth, trim=0.6in 0.6in 0.6in 0.6in, clip]{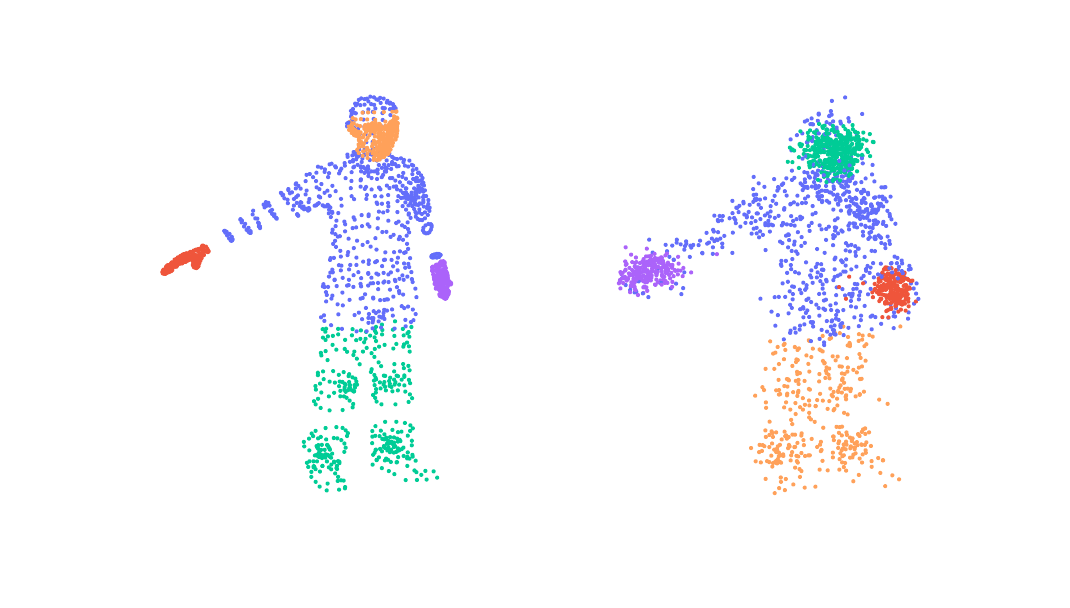}
    }
    \caption{Point clouds and the result of RSC. (a) shows the point cloud of a human (left) and the same point cloud rotated and with noise added (right). (b) shows the result of RSC applied to the clouds in (a) with $k = 5$ clusters. For presentation purposes, the clouds are rotated to share the same viewing angle.}
    \label{fig:alignment}
\end{figure}

The main contribution here is to avoid global matching between the two clouds, which can be computationally expensive and fail to capture the local structure of the clouds. Instead, we simultaneously cluster the clouds in a consistent manner via RSC and establish a correspondence between the matched clusters, thereby breaking the global matching problem into smaller subproblems that can be solved simultaneously. 

For a varying level of noise $\sigma^2$, we report the relative error in spectral norm between the estimated and true underlying rotation. We compare the proposed method with two alternative approaches: (i) an ICP \cite{besl1992method} pipeline starting with fast point feature histograms \cite{Rusu2009FastPF} and RANSAC algorithms for a coarse global correspondence estimate, followed by the generalized ICP algorithm \cite{segal2010generalized}, and (ii) the global Distance Profile Matching (DPM) between clouds $X$ and $Y$ as formulated in \cite{hur2025robust}, followed by the usual orthogonal Procrustes problem. Here, the noise level is converted to a signal-to-noise ratio (SNR) in decibels, defined as $\text{SNR}_{\text{db}} = 10 \log_{10}(\sigma_X^2 / \sigma^2)$, where $\sigma_X^2$ is the variance of the point cloud $X$ and $\sigma^2$ is the variance of the added noise. We repeat the comparison test over ten independent runs and report the mean relative error in Table~\ref{tab:method_errors}. First, we note that for a low to moderate presence of noise, our approach and the global DPM approach perform remarkably well, while the ICP pipeline quickly deteriorates as the noise level increases. For higher noise levels, our method is more robust than the global DPM approach. When using global DPM in a high-noise setup, there are no restrictions on the match, and a point can potentially be matched with every other point in the opposite cloud, which can lead to a less accurate match. In contrast, our method restricts the match to clusters that are more likely to be similar, thereby yielding a more accurate match.

\begin{table}[!htbp]
    \centering
    \begin{tabular}{|c|c|c|c|c|c|c|c|}
        \hline
        \diagbox{\textbf{Methods}}{$\textbf{SNR}_{\textbf{db}}$} 
        & \textbf{24.91} 
        & \textbf{18.89} 
        & \textbf{15.37} 
        & \textbf{10.90} 
        & \textbf{4.92} 
        & \textbf{2.91} \\ 
        \hline

        ICP Pipeline
        & 0.03285 
        & 0.37441  
        & 1.15744  
        & 1.67079 
        & 2.24804
        & 2.51093 \\ 
        \hline

        Global DPM 
        & {0.00288} 
        & {0.00722} 
        & {0.01460} 
        & {0.02889} 
        & 0.12395 
        & 0.25521 \\ 
        \hline

        Our Method 
        & 0.00668 
        & 0.02421 
        & 0.03134 
        & 0.09283  
        & {0.10219} 
        & {0.20242} \\ 
        \hline
    \end{tabular}
    \caption{Mean spectral-norm relative error between the estimated and true underlying rotations over ten independent runs, for varying noise levels $\sigma^2$ expressed in $\text{SNR}_{\text{db}}$.}
    \label{tab:method_errors}
\end{table}

\subsection{Non-Euclidean Application: RSC on High-Dimensional Stock Market Data} 
\label{subsec:stockmarket}
So far, we have focused on point clouds under the Euclidean metric to define the distance profiles. However, as profiles can be defined under any suitable similarity measure, we can also apply RSC beyond the usual Euclidean setting. In this section, we demonstrate the application of RSC to high-dimensional stock market data, where the similarity measure is based on the correlation between the returns of stocks.

Here, our goal is to cluster the top 50 companies from S\&P 500 (USA) and the top 50 companies from the Japanese stock market simultaneously, in order to analyze the similarity relations in the stock market between the two countries. For each company, we consider the daily closing price of its stock over a period of 5 years starting from the first day of 2020, which gives us a time series of length $T  = 1256$ (the number of trading days in 5 years), say, $(P(1), \ldots, P(T))$. We normalize the time series so that $P(1) = 1$ for all stocks. We do that to neutralize scaling problems that can occur from differences in currencies (US dollars to Japanese yen) and different starting points in value. We then compute the daily returns of each stock as the percentage change in price from the previous day, namely, $R(t) = (P(t) - P(t - 1)) / P(t - 1)$ for $t = 2, \ldots, T$ and $R(1)=0$. We take $(R(1),R(2), \ldots, R(T)) \in \R^{T}$ as the feature vector for each stock, which gives us a point cloud of size $n = 50$ in $\R^{T}$ for both the S\&P 500 and Japanese companies.

\begin{figure}[!htbp]
    \centering
    \includegraphics[width=0.495\linewidth]{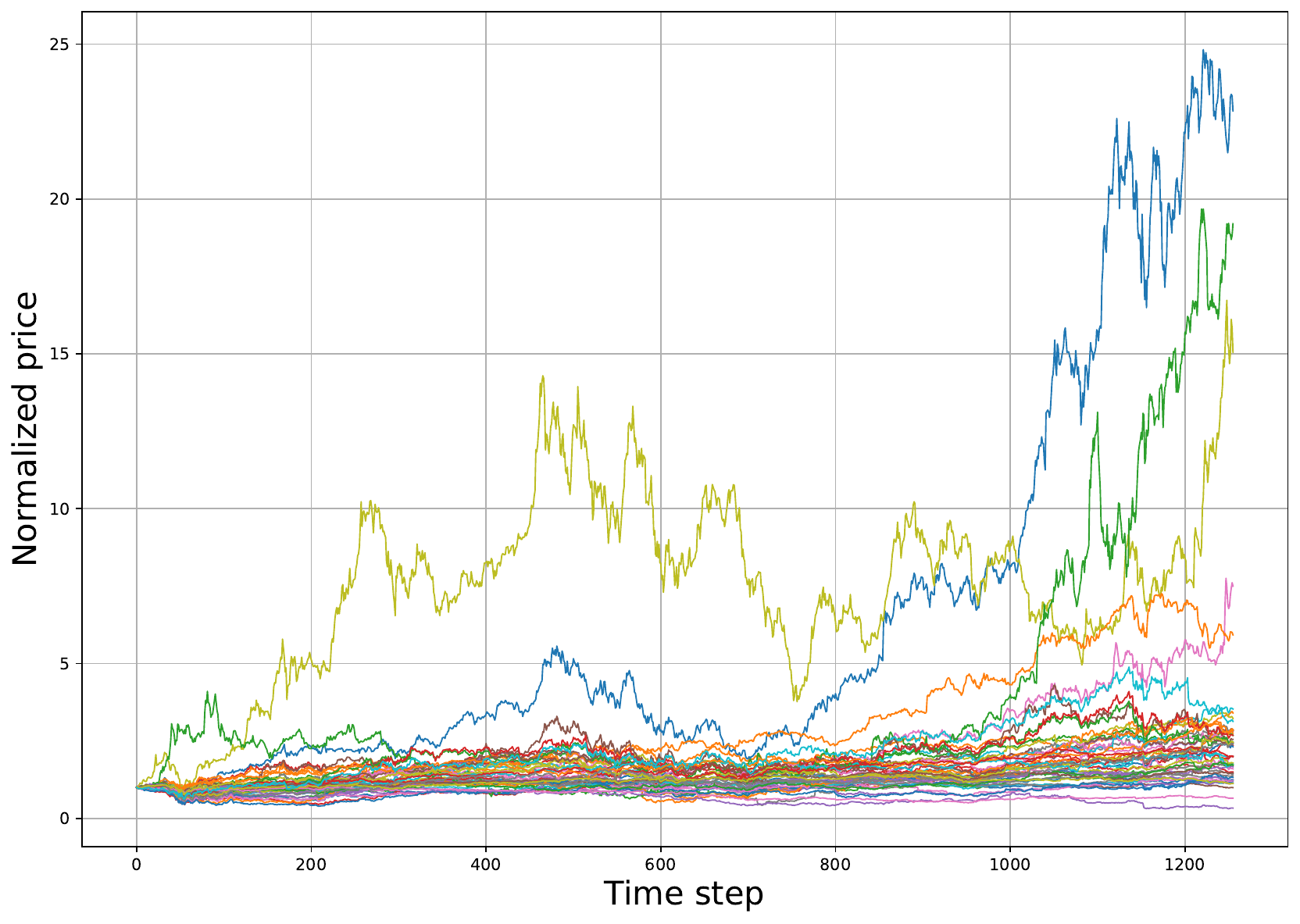}
    \includegraphics[width=0.495\linewidth]{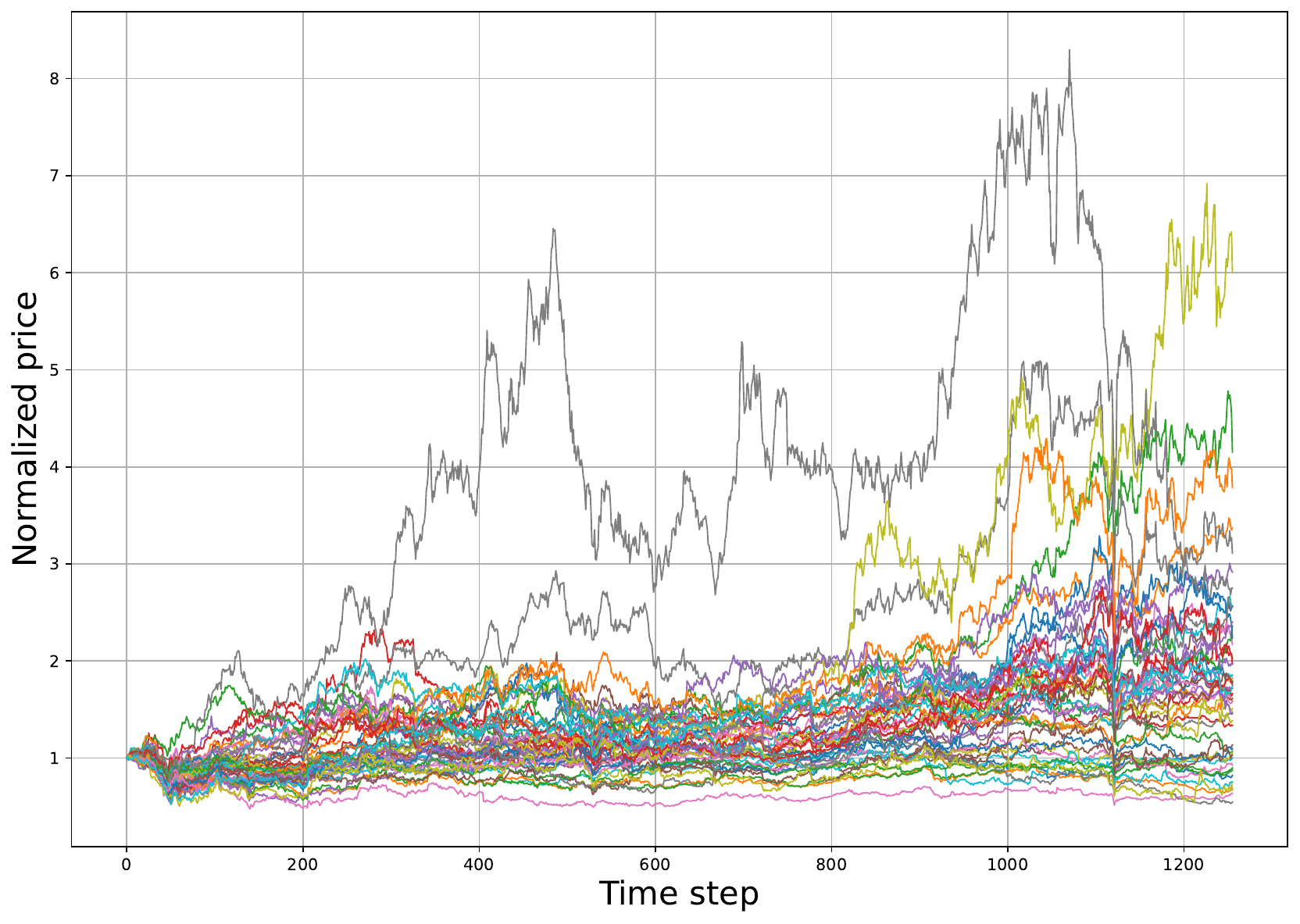}
    \caption{Normalized daily stock prices for the top 50 companies from the S\&P 500 (left) and top 50 Japanese companies by market cap (right) over five years from 2020.}
    \label{fig:stock_data}
\end{figure}

Let $R_i = (R_i(1), \ldots, R_i(T))$ for $i = 1, \ldots, n$ be the returns of the $n$ companies in S\&P 500. For $i, i' = 1, \ldots, n$, we define the covariance between the returns of companies $i$ and $i'$ as
\begin{equation*}
    \mathrm{Cov}(R_i, R_{i'}) = \frac{1}{T - 1} \sum_{t = 1}^{T} (R_i(t) - \bar{R}_i)(R_{i'}(t) - \bar{R}_{i'}),
\end{equation*}
where $\bar{R}_i = \frac{1}{T} \sum_{t = 1}^{T} R_i(t)$ is the mean return of company $i$. Then, the correlation between the returns of companies $i$ and $i'$ is defined as $\rho_{i i'} := \frac{\mathrm{Cov}(R_i, R_{i'})}{\sqrt{\mathrm{Var}(R_i) \mathrm{Var}(R_{i'})}}$, where $\mathrm{Var}(R_i) = \mathrm{Cov}(R_i, R_i)$ is the variance of the returns of company $i$. Finally, we define the similarity $d_{i i'}$ between companies $i$ and $i'$ as $d_{i i'} = \sqrt{2(1 - \rho_{i i'})}$, which we treat as a distance measure between the two companies. Note that when $R_i$ and $R_{i'}$ are positively correlated, $d_{i i'}$ is small. We define the similarity in the same way for the Japanese companies. This distance is then transformed into a similarity score using an RBF kernel as before, from which the similarity matrices $K_X$ and $K_Y$ are constructed. 

Since the correlation-based RBF graphs in this example are connected, the exact block-constancy guarantee of Theorem~\ref{thm:optimal_coupling_projection_bound} is best interpreted as an idealized limiting case. The relevant mechanism here is Proposition~\ref{prop:proj_second_bound}: LapOT suppresses oscillatory components of the coupling over the two market similarity graphs, so broad sector-level structure can appear as an approximately low-rank or block-structured coupling.

Unlike the previous examples, we construct the cost matrix $C$ for LapOT using different similarity profiles than the ones used to construct the similarity matrices $K_X$ and $K_Y$. This is because we want to capture different financial characteristics of the stocks in the two countries. Specifically, we take the beta of each stock as a measure of its systematic risk, a measure of a stock's volatility in relation to the overall market. The beta of stock $i$ is $\beta_i = \mathrm{Cov}(R_i, R_{\text{market}}) / \mathrm{Var}(R_{\text{market}})$, where $R_{\text{market}}$ is the return of the market index. The cost $C_{i j}$ is then set to the Wasserstein-1 distance between the similarity profiles that are picked so that $W_1(\mu_i,\nu_j)$ is an upper bound on $|\beta_i - \beta_j|$. This choice of cost encourages matching stocks with similar systematic risk profiles. The formulation of the specific choice for $\mu_i$ and $\nu_j$ can be found in Section \ref{appendix:stock market} in the appendix.

\begin{figure}[!bp]
    \centering
    \includegraphics[trim=0.1cm 0.2cm 0cm 0.2cm, clip=true, width=0.9\textwidth]{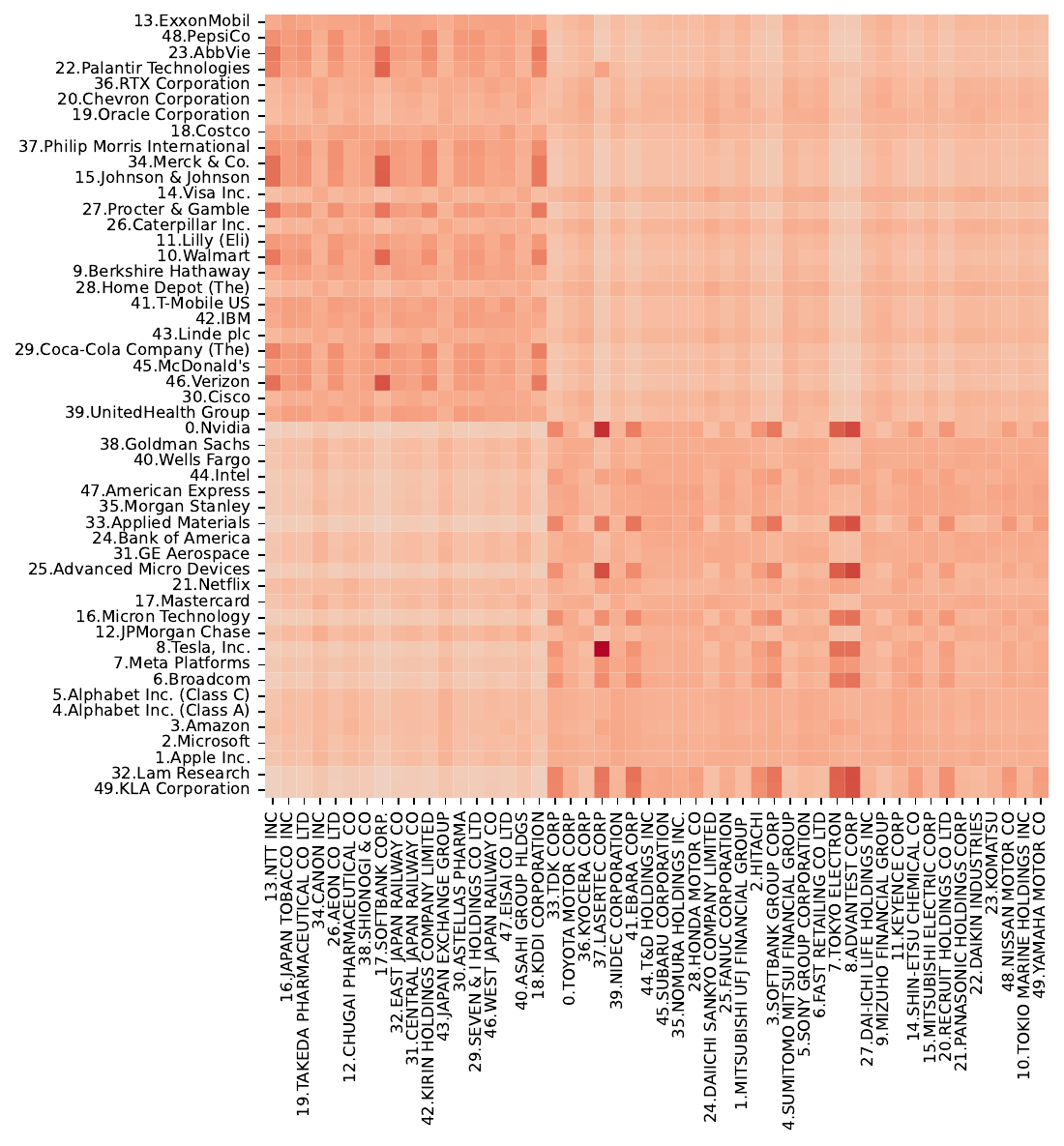}
    \caption{The solution of LapOT applied to the US-Japan stock market data. We permuted the rows and columns of the optimal coupling to better visualize its low-rank structure.}
    \label{fig: OTQuad US-Japan}
\end{figure}

With this setup, we expect LapOT to produce a coupling that aligns companies with similar sector-based behaviors (captured by $L_X, L_Y$) and similar market volatility profiles (captured by $C$). In Figure \ref{fig: OTQuad US-Japan}, we show the low-rank structure of the optimal coupling solution to \eqref{eq:LapOT}, indicating broad cluster similarities between the two countries. To complement the low-rank structure observed in Figure~\ref{fig: OTQuad US-Japan}, we numerically evaluate the upper bound in Proposition~\ref{prop:proj_second_bound} for increasing values of $\ell$ and $h$. Figure~\ref{fig:theoretical bound stock} provides direct numerical support for the proposition: the theoretical bound uniformly controls the projection error and captures its decay as the dimensions of the retained low-frequency eigenspaces increase.

In the appendix, Figures \ref{fig:sp500 clusters} and \ref{fig:Japan clusters} show the results of applying RSC to the companies from the USA and Japan, respectively, showing a consistent pattern in the results.

\begin{figure}[!htbp]
    \centering
    \includegraphics[trim=0.1cm 0.2cm 0cm 0.2cm, clip=true, width=0.5\textwidth]{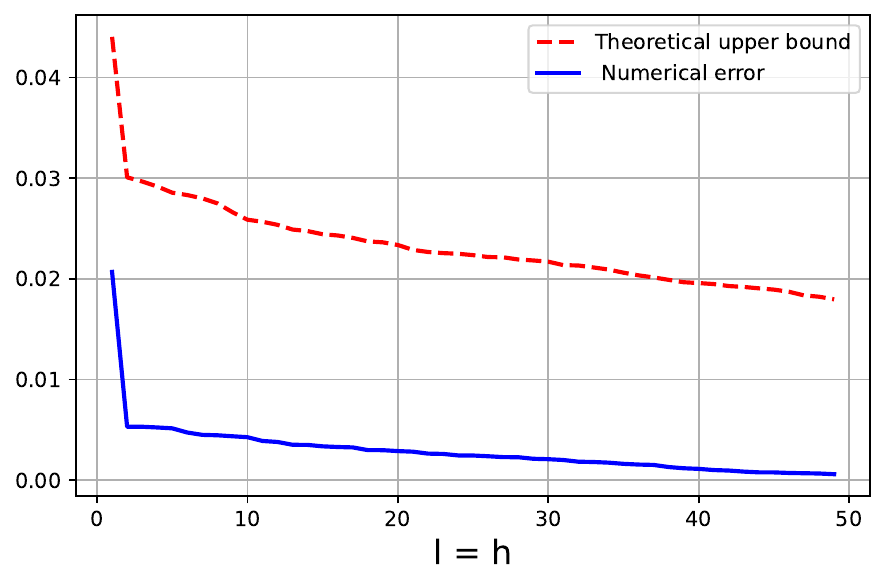}
    \caption{Comparison between the theoretical upper bound and the projection error $\|\pi^\star - P_{X,\ell}\pi^\star P_{Y,h}\|_{\mathrm{F}}$ of the optimal coupling $\pi^\star$, evaluated for $\ell=h=1,\ldots,n-1$.
    }
    \label{fig:theoretical bound stock}
\end{figure}

\bibliography{bib}
\bibliographystyle{plain}

\appendix
\section{Optimization of Low-Rank LapOT}
\label{sec:lr_LapOT_optimization}
This section discusses the optimization scheme for solving the low-rank version of the LapOT problem \eqref{eq:lr_LapOT}. As noticed in \cite{scetbon2021low}, the division by $g$ in the objective function of \eqref{eq:lr_LapOT} can lead to numerical instability. To mitigate this issue, \cite{scetbon2021low} suggests restricting the search domain to 
\begin{equation*}
    \cC(a, b, r, \alpha) := \{(U, V, g) \in \cC(a, b, r) : g \ge \alpha 1_r\},
\end{equation*}
where $\alpha > 0$ is a user-defined stability parameter. This restriction prevents any entry of $g$ from being too close to zero. Hence, a stable formulation of the low-rank LapOT problem \eqref{eq:lr_LapOT} is given by
\begin{equation}
    \label{eq:lr_LapOT_stable}
    \min_{(U, V, g) \in \cC(a, b, r, \alpha)} f(U, V, g),
\end{equation}
where $f(U, V, g)$ is the objective function in \eqref{eq:lr_LapOT}, namely,
\begin{equation*}
    \begin{split}
        f(U, V, g) 
        & = \langle U \mathrm{diag}(1_r / g) V^\top, C \rangle \\
        & \quad + \lambda_x \langle U \mathrm{diag}(1_r / g) V^\top, L_X U \mathrm{diag}(1_r / g) V^\top \rangle + \lambda_y \langle U \mathrm{diag}(1_r / g) V^\top, U \mathrm{diag}(1_r / g) V^\top L_Y \rangle \\
        & \quad - \lambda H(U) - \lambda H(V) - \lambda H(g).
    \end{split}
\end{equation*}
Algorithm \ref{alg:lot_quad} presented at the end of this section summarizes the optimization scheme for solving \eqref{eq:lr_LapOT_stable}.

To understand Algorithm \ref{alg:lot_quad}, let us start by following \cite{scetbon2021low} to rewrite the feasible set $\cC(a, b, r, \alpha)$ as the intersection of two sets $\cC_1(a, b, r, \alpha)$ and $\cC_2(r)$, where 
\begin{align*}
    \cC_1(a, b, r, \alpha) & \triangleq \{(U, V, g) \in \R_+^{n \times r} \times \R_+^{m \times r} \times \R_+^r : U 1_r = a, V 1_r = b, g \ge \alpha 1_r\}, \\
    \cC_2(r) & \triangleq \{(U, V, g) \in \R_+^{n \times r} \times \R_+^{m \times r} \times \R_+^r : U^\top 1_n = V^\top 1_m = g\},
\end{align*}
which allows us to apply Dykstra's algorithm for projection onto the intersection of convex sets. As in \cite{scetbon2021low}, we implement mirror descent under the KL divergence to solve \eqref{eq:lr_LapOT_stable}. More precisely, for $x_k \triangleq (U_k, V_k, g_k)$ and step size $\gamma_k > 0$ at iteration $k$, we update $x_{k + 1}$ by solving the following optimization problem:
\begin{equation*}
    x_{k + 1} = \argmin_{x \in \cC(a, b, r, \alpha)} \langle \nabla f(x_k), x \rangle + \frac{1}{\gamma_k} \mathsf{KL}(x \, \| \, x_k),
\end{equation*}
which is equivalent to 
\begin{equation}
    \label{eq:mirror_descent_update}
    x_{k + 1} = \argmin_{x \in \cC(a, b, r, \alpha)} \mathsf{KL}\left(x \, \| \, x_k \odot \exp(-\gamma_k \nabla f(x_k))\right),
\end{equation}
where $\odot$ denotes the element-wise product, and the exponential is taken element-wise. Here, the gradient $\nabla f(U, V, g)$ is given by
\begin{align*}
    \nabla_U f & = M V \mathrm{diag}(1_r / g) + \lambda \log(U), \\
    \nabla_V f & = M^\top U \mathrm{diag}(1_r / g) + \lambda \log(V ), \\
    \nabla_g f & = -\frac{\mathrm{diag}(U^\top M V)}{g^2} + \lambda \log(g),
\end{align*}
where $g^2$ is the vector obtained by squaring each entry of $g$, and $M$ is given by
\begin{equation*}
    M = 2 \lambda_x L_X U \mathrm{diag}(1_r / g) V^\top + 2 \lambda_y U \mathrm{diag}(1_r / g) V^\top L_Y + C.
\end{equation*}
The projection onto $\cC(a, b, r, \alpha)$ under the KL divergence in \eqref{eq:mirror_descent_update} can be efficiently computed using Dykstra's algorithm, which alternates between projecting onto $\cC_1(a, b, r, \alpha)$ and $\cC_2(r)$. To see this, define
\begin{align*}
    \tilde{U}_k & \triangleq U_k \odot \exp(-\gamma_k \nabla_U f(U_k, V_k, g_k)) = U_k \odot \exp\left(-\gamma_k \left(M_k V_k \mathrm{diag}(1_r / g_k) + \lambda \log(U_k)\right)\right), \\
    \tilde{V}_k & \triangleq V_k \odot \exp(-\gamma_k \nabla_V f(U_k, V_k, g_k)) = V_k \odot \exp\left(-\gamma_k \left(M_k^\top U_k \mathrm{diag}(1_r / g_k) + \lambda \log(V_k)\right)\right), \\
    \tilde{g}_k & \triangleq g_k \odot \exp(-\gamma_k \nabla_g f(U_k, V_k, g_k)) = g_k \odot \exp\left(-\gamma_k \left(- \frac{\mathrm{diag}(U_k^\top M_k V_k)}{g_k^2} + \lambda \log(g_k)\right)\right),
\end{align*}
where $M_k$ is defined as
\begin{equation*}
    M_k = 2 \lambda_x L_X U_k \mathrm{diag}(1_r / g_k) V_k^\top + 2 \lambda_y U_k \mathrm{diag}(1_r / g_k) V_k^\top L_Y + C.
\end{equation*}
Hence, the mirror descent update in \eqref{eq:mirror_descent_update} can be rewritten as
\begin{equation}
    \label{eq:mirror_descent_update_simplified}
    (U_{k + 1}, V_{k + 1}, g_{k + 1}) = \argmin_{(U, V, g) \in \cC(a, b, r, \alpha)} \mathrm{KL}\left((U, V, g) \, \| \, (\tilde{U}_k, \tilde{V}_k, \tilde{g}_k)\right).
\end{equation}

\paragraph{Dykstra's Algorithm.}
Changing the variable names to $(\xi^{(1)}, \xi^{(2)}, \xi^{(3)}) \in \R_+^{n \times r} \times \R_+^{m \times r} \times \R_+^r$ in \eqref{eq:mirror_descent_update_simplified}, we now derive the Dykstra's algorithm to solve the following:
\begin{equation}
    \label{eq:projection_problem}
    \argmin_{(U, V, g) \in \cC(a, b, r, \alpha)} \mathrm{KL}\left((U, V, g) \, \| \, (\xi^{(1)}, \xi^{(2)}, \xi^{(3)})\right).
\end{equation}
Algorithm \ref{alg:dykstra} summarizes Dykstra's algorithm for solving \eqref{eq:projection_problem}, where $r(\cdot)$ and $c(\cdot)$ denote the row and column sum operators, respectively. To see this, recall that Dykstra's algorithm \cite{dykstra1983algorithm} iterates as follows:
\begin{align*}
    (\xi_{2 k + 1}^{(1)}, \xi_{2 k + 1}^{(2)}, \xi_{2 k + 1}^{(3)}) & = \argmin_{(U, V, g) \in \cC_1(a, b, r, \alpha)} \mathrm{KL}\left((U, V, g) \, \| \, (\xi_{2 k}^{(1)} \odot q_{2 k - 1}^{(1)}, \xi_{2 k}^{(2)} \odot q_{2 k - 1}^{(2)}, \xi_{2 k}^{(3)} \odot q_{2 k - 1}^{(3)})\right), \\
    (q_{2 k + 1}^{(1)}, q_{2 k + 1}^{(2)}, q_{2 k + 1}^{(3)}) & = \left(\frac{\xi_{2 k}^{(1)} \odot q_{2 k - 1}^{(1)}}{\xi_{2 k + 1}^{(1)}}, \frac{\xi_{2 k}^{(2)} \odot q_{2 k - 1}^{(2)}}{\xi_{2 k + 1}^{(2)}}, \frac{\xi_{2 k}^{(3)} \odot q_{2 k - 1}^{(3)}}{\xi_{2 k + 1}^{(3)}}\right), \\
    (\xi_{2 k + 2}^{(1)}, \xi_{2 k + 2}^{(2)}, \xi_{2 k + 2}^{(3)}) & = \argmin_{(U, V, g) \in \cC_2(r)} \mathrm{KL}\left((U, V, g) \, \| \, (\xi_{2 k + 1}^{(1)} \odot q_{2 k}^{(1)}, \xi_{2 k + 1}^{(2)} \odot q_{2 k}^{(2)}, \xi_{2 k + 1}^{(3)} \odot q_{2 k}^{(3)})\right), \\
    (q_{2 k + 2}^{(1)}, q_{2 k + 2}^{(2)}, q_{2 k + 2}^{(3)}) & = \left(\frac{\xi_{2 k + 1}^{(1)} \odot q_{2 k}^{(1)}}{\xi_{2 k + 2}^{(1)}}, \frac{\xi_{2 k + 1}^{(2)} \odot q_{2 k}^{(2)}}{\xi_{2 k + 2}^{(2)}}, \frac{\xi_{2 k + 1}^{(3)} \odot q_{2 k}^{(3)}}{\xi_{2 k + 2}^{(3)}}\right).
\end{align*}
Propositions 2 and 3 of \cite{scetbon2021low} state the closed-form solutions for the projections onto $\cC_1(a, b, r, \alpha)$ and $\cC_2(r)$, where \eqref{eq:dykstra_projection1} and \eqref{eq:dykstra_projection2} provide the respective formulas. Algorithm \ref{alg:dykstra} iterates these projections until the row sums of $\xi_{2 k + 2}^{(1)}$ and $\xi_{2 k + 2}^{(2)}$ are sufficiently close to $a$ and $b$, respectively.

\begin{algorithm}[!htbp]
    \caption{$\mathsf{Dykstra}(\xi^{(1)}, \xi^{(2)}, \xi^{(3)}, a, b, \alpha, \delta)$}
    \label{alg:dykstra}
    \begin{algorithmic}[1]
        \Input $\xi^{(1)} \in \R_+^{n \times r}$, $\xi^{(2)} \in \R_+^{m \times r}$, $\xi^{(3)} \in \R_+^r$ with positive entries, marginal weights $a \in \Delta_n$ and $b \in \Delta_m$.
        \Input stability parameter $\alpha > 0$, desired accuracy $\delta > 0$.
        \State Initialize $(\xi_0^{(1)}, \xi_0^{(2)}, \xi_0^{(3)}) = (\xi^{(1)}, \xi^{(2)}, \xi^{(3)})$ and $(q_{-1}^{(1)}, q_{-1}^{(2)}, q_{-1}^{(3)}) = (q_0^{(1)}, q_0^{(2)}, q_0^{(3)}) = (1_{n \times r}, 1_{m \times r}, 1_r)$. 
        \State Set $k = 0$.
        \While{$||\xi_{2 k}^{(1)} 1_r - a\|_1 + ||\xi_{2 k}^{(2)} 1_r - b\|_1 \ge \delta$}
            \State Update
            \begin{equation}
                \label{eq:dykstra_projection1}
                \begin{aligned}
                    \xi_{2 k + 1}^{(1)} & = \mathrm{diag}\left(\frac{a}{r(\xi_{2 k}^{(1)} \odot q_{2 k - 1}^{(1)})}\right) (\xi_{2 k}^{(1)} \odot q_{2 k - 1}^{(1)}), \\
                    \xi_{2 k + 1}^{(2)} & = \mathrm{diag}\left(\frac{b}{r(\xi_{2 k}^{(2)} \odot q_{2 k - 1}^{(2)})}\right) (\xi_{2 k}^{(2)} \odot q_{2 k - 1}^{(2)}), \\
                    \xi_{2 k + 1}^{(3)} & = \max\left(\xi_{2 k}^{(3)} \odot q_{2 k - 1}^{(3)}, \alpha 1_r\right).
                \end{aligned}
            \end{equation}
            \State Update
            \begin{equation}
                \label{eq:dykstra_projection2}
                \begin{aligned}
                    \xi_{2 k + 2}^{(3)} & = \left(\xi_{2 k + 1}^{(3)} \odot q_{2 k }^{(3)} \odot c(\xi_{2 k + 1}^{(1)} \odot q_{2 k}^{(1)}) \odot c(\xi_{2 k + 1}^{(2)} \odot q_{2 k}^{(2)})\right)^{1 / 3}, \\
                    \xi_{2 k + 2}^{(1)} & = (\xi_{2 k + 1}^{(1)} \odot q_{2 k}^{(1)}) \odot \mathrm{diag}\left(\frac{\xi_{2 k + 2}^{(3)}}{c(\xi_{2 k + 1}^{(1)} \odot q_{2 k}^{(1)})}\right), \\
                    \xi_{2 k + 2}^{(2)} & = (\xi_{2 k + 1}^{(2)} \odot q_{2 k}^{(2)}) \odot \mathrm{diag}\left(\frac{\xi_{2 k + 2}^{(3)}}{c(\xi_{2 k + 1}^{(2)} \odot q_{2 k}^{(2)})}\right).
                \end{aligned}
            \end{equation}
            \State Update $k \leftarrow k + 1$.
        \EndWhile
        \Output $(\xi_{2 k}^{(1)}, \xi_{2 k}^{(2)}, \xi_{2 k}^{(3)})$.
    \end{algorithmic}
\end{algorithm}

\paragraph{Adaptive Step Size.} We use the adaptive step size scheme proposed by \cite{scetbon2022low} to avoid overflowing in the exponent terms used as input for Dykstra's algorithm. Hence, when solving \eqref{eq:mirror_descent_update} at iteration $k$, we set the step size $\gamma_k$ to be
\begin{equation*}
    \gamma_k = \frac{\gamma}{\|\nabla f(x_k)\|_\infty^2},
\end{equation*}
where $\gamma > 0$ is a user-defined parameter. Following the recommendation of \cite{scetbon2022low}, we set $\gamma \in [1, 10]$ for the initialization.

\begin{algorithm}[!htbp]
    \caption{Low-Rank LapOT}
    \label{alg:lot_quad}
    \begin{algorithmic}[1]
        \Input matching cost $C \in \R^{n \times m}$, marginal weights $a \in \Delta_n$ and $b \in \Delta_m$.
        \Input Graph Laplacians $L_X \in \R^{n \times n}$ and $L_Y \in \R^{m \times m}$, hyperparameters $\lambda_x, \lambda_y, \lambda$.
        \Input rank $r \in \N$, stability parameter $\alpha > 0$, desired accuracy $\delta > 0$ for Dykstra's algorithm.
        \Input number of iterations $T \in \N$.
        \State Initialize $(U_0, V_0, g_0) \in \R_+^{n \times r} \times \R_+^{m \times r} \times \R_+^{r}$ with positive entries.
        \For{$k = 0, \ldots, T - 1$}
            \State Set the step size $\gamma_k = \frac{\gamma}{\|\nabla f(U_k, V_k, g_k)\|_\infty^2}$.
            \State Compute
            \begin{equation*}
                M_k = 2 \lambda_x L_X U_k \mathrm{diag}(1_r / g_k) V_k^\top + 2 \lambda_y U_k \mathrm{diag}(1_r / g_k) V_k^\top L_Y + C.
            \end{equation*}
            \State Compute
            \begin{align*}
                \tilde{U}_k & = U_k \odot \exp\left(-\gamma_k \left(M_k V_k \mathrm{diag}(1_r / g_k) + \lambda \log(U_k)\right)\right), \\
                \tilde{V}_k & = V_k \odot \exp\left(-\gamma_k \left(M_k^\top U_k \mathrm{diag}(1_r / g_k) + \lambda \log(V_k)\right)\right), \\
                \tilde{g}_k & = g_k \odot \exp\left(-\gamma_k \left(- \frac{\mathrm{diag}(U_k^\top M_k V_k)}{g_k^2} + \lambda \log(g_k)\right)\right).
            \end{align*}
            \State Update $(U_{k + 1}, V_{k + 1}, g_{k + 1}) = \mathsf{Dykstra}(\tilde{U}_k, \tilde{V}_k, \tilde{g}_k, a, b, \alpha, \delta)$.
        \EndFor
    
        \Return \((U_T, V_T, g_T)\).
    \end{algorithmic}
\end{algorithm}

\section{Further Details of the Stock Market Application}
\label{appendix:stock market}
\paragraph{Similarity Profiles.}
For the data introduced in Section~\ref{subsec:stockmarket},
between every two stocks $i$ and $j$, we set the similarity between them to be
\[
    \bar{\rho}_{ij} = \frac{\mathrm{Cov}(R_i, \frac{P_j}{P_{\mathrm{market}}}R_j)}{\mathrm{Var}(R_{\mathrm{market}})},
\]
where $P_{\mathrm{market}}$ is the price of the market index, and $R_{\mathrm{market}}$ is the return of the market index. The distributions induced by these similarity profiles are then $\mu_i = \frac{1}{n}\sum_{k=1}^{n}\delta_{\bar{\rho}_{ik}}$ and $\nu_j = \frac{1}{n}\sum_{k=1}^{n}\delta_{\bar{\rho'}_{jk}}$. Now, let $C_{ij} = W_1(\mu_i, \nu_j)$, and assume $\bar{\rho}_{i1} \leq\bar{\rho}_{i2} \leq \dots \leq \bar{\rho}_{in}$ and $\bar{\rho'}_{j1} \leq \bar{\rho'}_{j2} \leq \dots \leq \bar{\rho'}_{jn}$, without loss of generality. Then, we have
\[
    \begin{split}
        W_1(\mu_i, \nu_j) 
        & = \frac{1}{n}\sum_{k=1}^n |\bar{\rho}_{ik} - \bar{\rho'}_{jk}| \\
        & \geq \frac{1}{n}|\sum_{k=1}^n \bar{\rho}_{ik} - \bar{\rho'}_{jk}| \\
        & = |\frac{1}{n}\sum_{k=1}^n\frac{\mathrm{Cov}(R_i, \frac{P_k}{P_{\mathrm{market}}}R_k)}{\mathrm{Var}(R_{\mathrm{market}})} - \frac{1}{n}\sum_{k=1}^n\frac{\mathrm{Cov}(R'_j, \frac{P'_k}{P'_{\mathrm{market}}}R'_k)}{\mathrm{Var}(R'_{\mathrm{market}})}| \\
        & = |\frac{\mathrm{Cov}(R_i, R_{\mathrm{market}})}{\mathrm{Var}(R_{\mathrm{market}})} - \frac{\mathrm{Cov}(R'_j, R'_{\mathrm{market}})}{\mathrm{Var}(R'_{\mathrm{market}})}| \\
        & = |\beta_{i} - \beta'_{j}|,
    \end{split}
\]
which, in practice, appears to be a pretty tight upper bound. This means that by choosing such $C$, we connect the cost of transportation between the American stock $i$ and the Japanese stock $j$ with the difference between their betas \cite{brealey2011principles}.

\paragraph{Additional Results of RSC}
Here, we provide visualizations of the results of using our method on financial data. In Figures \ref{fig:sp500 clusters} and \ref{fig:Japan clusters}, we show the cluster structure obtained by using RSC on the data introduced in Section~\ref{subsec:stockmarket}. 

\begin{figure}[!htbp]
    \centering
    \includegraphics[trim=0.1cm 0.2cm 0cm 0.2cm, clip=true, width=0.9\textwidth]{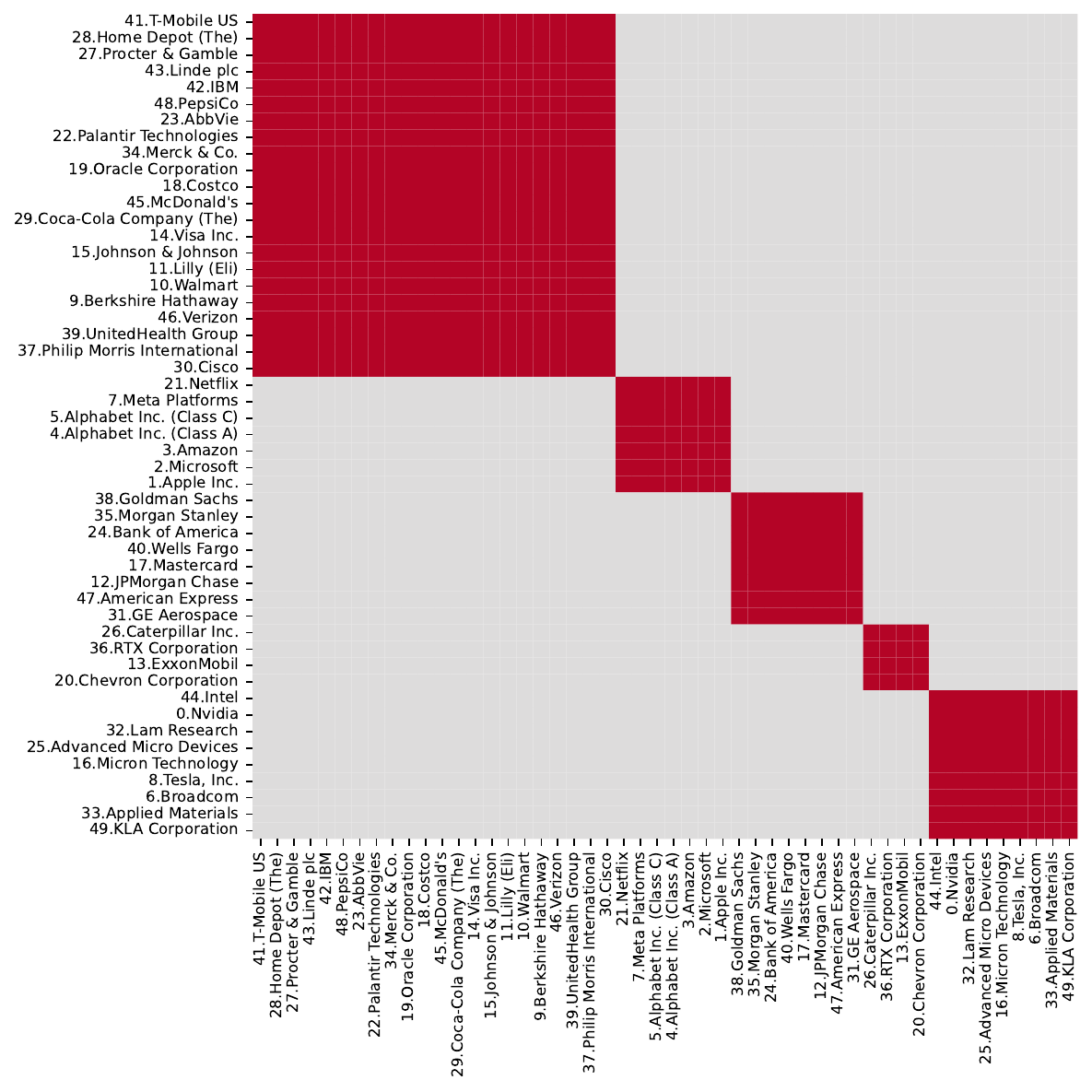}
    \caption{The cluster structure of the top 50 stocks from S\&P 500 found by the RSC. We see the rough partition into $\{$\textbf{big tech, semiconductor \& deep tech infrastructure, finance, energy, pharma \& retail \& the rest}$\}$.}
    \label{fig:sp500 clusters}
\end{figure}

\begin{figure}[!htbp]
    \centering
    \includegraphics[trim=0.1cm 0.2cm 0cm 0.2cm, clip=true, width=0.9\textwidth]{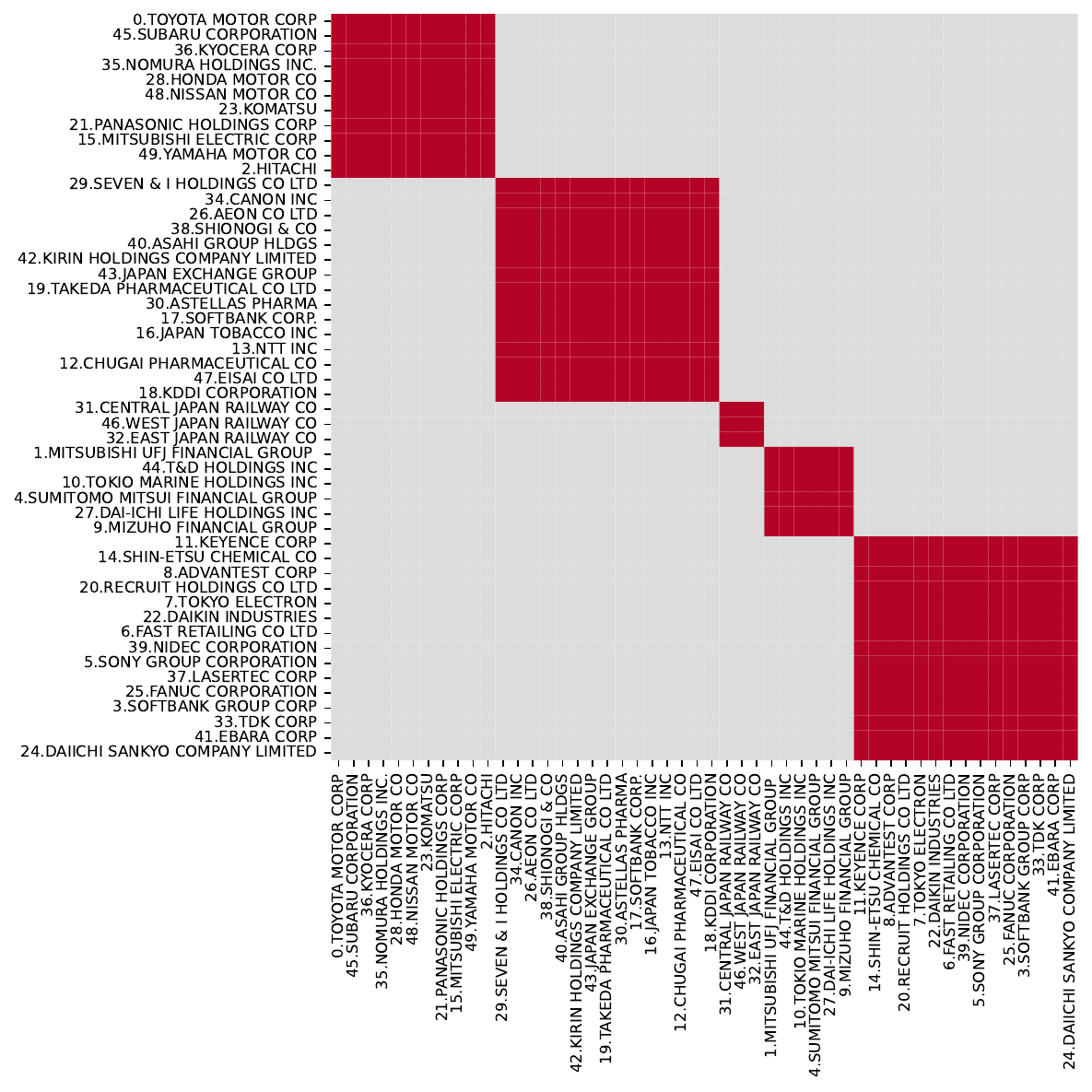}
    \caption{The cluster structure of the top 50 Japanese stocks found by the RSC. We see the rough partition into $\{$\textbf{automotive \& mobility, Japan Railway, finance, semiconductor \& electronics manufacturing \& technology, pharma \& retail \& the rest}$\}$.}
    \label{fig:Japan clusters}
\end{figure}

\end{document}